\documentclass{article}

\usepackage{jmlr2e,times}

\usepackage{amsmath,amsfonts,bm}

\def\eqref#1{equation~\ref{#1}}

\def\1{\bm{1}}

\DeclareMathAlphabet{\mathsfit}{\encodingdefault}{\sfdefault}{m}{sl}
\SetMathAlphabet{\mathsfit}{bold}{\encodingdefault}{\sfdefault}{bx}{n}

\newcommand{\E}{\mathbb{E}}

\newcommand{\R}{\mathbb{R}}

\usepackage{multirow}
\usepackage{makecell}

\usepackage[utf8]{inputenc} 
\usepackage[T1]{fontenc}    
\usepackage{hyperref}       
\usepackage{url}            
\usepackage{booktabs}       
\usepackage{amsfonts}       
\usepackage{nicefrac}       
\usepackage{microtype}      
\usepackage{amsmath}
\usepackage{subfigure}
\usepackage[ruled,vlined]{algorithm2e}
\usepackage{algorithmic}
\usepackage{fancybox}
\usepackage[english]{babel}
\usepackage{graphicx, caption,  wrapfig, subfigure}
\usepackage{xcolor}

\newtheorem{thm}{Theorem}

\newtheorem{lem}{Lemma}

\newtheorem{ass}{Assumption}

\def \D {\mathcal{D}}
\def \A {\mathcal{A}}
\def \P {\mathcal{P}}
\def \B {\mathcal{B}}

\def \I {\mathbb{I}}

\def \X {\mathcal{X}}
\def \F {\mathcal{F}}
\def \Z {\mathcal{Z}}
\def \Y {\mathcal{Y}}
\def\x{{\bf{x}}}
\def\y{{\bf{y}}}
\def\z{{\bf{z}}}
\def\w{{\bf{w}}}
\def\v{{\bf{v}}}
\def\d{{\bf{d}}}

\def \u {{\bf{u}}}

\def \E {\mathbb{E}}
\def \R {\mathbb{R}}
\def \V {\mathbb{V}}

\def \p {{\bf{p}}}

\jmlrheading{1}{2022}{1-56}{4/00}{10/00}{meila00a}{Qi, Ardeshir, Xu, Yang}

\ShortHeadings{Fairness via Adversarial Attribute Neighbourhood Robust Learning}{Qi, Ardeshir, Xu, Yang} 
\title{Fairness via Adversarial Attribute Neighbourhood Robust Learning}

\author{\name{Qi  Qi$^{\dagger}$}\email{qi-qi@uiowa.edu}\\
\name{Shervin Ardeshir$^{\ddagger}$}\email{shervina@netflix.com}\\
\name{Yi Xu$^\nmid$}\email{yxu@dlut.edu.cn}\\
\name{Tianbao Yang$^\nshortmid$}\email{tianbao-yang@tamu.edu}\\
   \addr$^\dagger$ Department of Computer Science, The University of Iowa, Iowa City, IA 52242, USA \\
    \addr$^\ddagger$ Netflix,100 Winchester Circle, Los Gatos, CA 95032, USA\\
     \addr$^\nmid$ School of Artificial Intelligence, Dalian University of Technology, Dalian, Liaoning 116024, China\\
 \addr$^\nshortmid$ Department of Computer Science \& Engineering,
Texas A\&M University, College Station, TX 77843, USA \\
}

\begin{document}
\maketitle

\begin{abstract}
Improving fairness between privileged and less-privileged sensitive attribute groups (e.g, {race, gender}) has attracted lots of attention. 
To enhance the model performs uniformly well in different sensitive attributes, we propose a principled \underline{R}obust \underline{A}dversarial \underline{A}ttribute \underline{N}eighbourhood (RAAN) loss to debias the classification head and promote a fairer representation distribution across different sensitive attribute groups. The key idea of RAAN is to mitigate the differences of biased representations between different sensitive attribute groups by assigning each sample an adversarial robust weight, which is defined on the representations of adversarial attribute neighbors, i.e, the samples from different protected groups. To provide efficient optimization algorithms, we cast the RAAN into a sum of coupled compositional functions and propose a stochastic adaptive (Adam-style) and non-adaptive (SGD-style) algorithm framework SCRAAN with provable theoretical guarantee. 
Extensive empirical studies on fairness-related benchmark datasets verify the effectiveness of the proposed method.

\end{abstract}

\section{Introduction}

With the excellent performance, machine learning methods have penetrated into many fields and brought impact into our daily lifes, such as the recommendation~\citep{lin2022quantifying, zhang2021measuring}, sentiment analysis~\citep{kiritchenko2018examining, adragna2020fairness} and facial detection systems~\citep{buolamwini2018gender}. Due to the existing bias and confounding factors in the training data~\citep{fabbrizzi2022survey, torralba2011unbiased}, model predictions are often correlated with sensitive attributes, e.g, {race, gender}, which leads to undesirable outcomes. Hence, fairness concern has become increasingly prominent. For example, the job recommendation system recommends lower wage jobs more likely to women than men~\citep{zhang2021measuring}.~\cite{buolamwini2018gender} proposed an intersectional approach that quantitatively show that three commercial gender classifiers, proposed by Microsoft, IBM and Face++, have higher error rate for the {darker-skinned} populations.

To alleviate the effect of spurious correlations\footnote{misleading heuristics that work for most training examples but do not always hold.} between the sensitive attribute groups and prediction, many bias mitigation methods have been proposed to learn a debiased representation distribution at encoder level by taking the advantage of the adversarial learning~\citep{wang2019repairing, wadsworth2018achieving, edwards2015censoring, elazar2018adversarial}, causal inference~\citep{singh2020don, kim2019learning} and invariant risk minimization~\citep{ adragna2020fairness, arjovsky2019invariant}. Recently, in order to further improve the performance and reduce computational costs for large-scale data training,  learning a classification head using the representation of pretrained models have been widely used for different tasks. Taking image classification for example, the downstream tasks are trained by finetuning the classification head of ImageNet pretrained ResNet~\citep{he2016deep} model~\citep{qi2020attentional,kang2019decoupling}. However, the pretrained model may introduce the undesiarable bias for the downstreaming tasks. 
Debiasing the encoder of pretrained models to have fairer representations by retraining is time-consuming and compuatational expensive. Hence, debiasing the classification head on biased representations is also of great importance. 

In this paper, we raise two research questions: {\it Can we improve the fairness of the classification head on a biased representation space? Can we further reduce the bias in the representation space?} We give affirmative answers by proposing a \underline{R}obust \underline{A}dversarial \underline{A}ttribute \underline{N}eighborhood (RAAN) loss. Our work is inspired by the RNF method~\citep{du2021fairness}, which averages the representation of sample pairs from different protected groups to alleviate the undesirable correlation between sensitive information and specific class labels. But unlike RNF, RAAN obtains fairness-promoting adversarial robust weights by exploring the \underline{A}dversarial \underline{A}ttribute \underline{N}eighborhood (AAN) representation structure for each sample to mitigate the differences of biased sensitive attribute representations. To be more specific, the adversarial robust weight for each sample is the aggregation of the pairwise robust weights defined on the representation similarity between the sample and its AAN. Hence, the greater the representation similarity, the more uniform the distribution of protected groups in the representation space. Therefore, by promoting higher pairwise weights for larger similarity pairs, RAAN is able to mitigate the discrimination of the biased senstive attribute representations and promote a fairer classification head. When the representation is fixed, RAAN is also applicable to debiasing the classification head only. 

We use a toy example of binary classification to express the advantages of RAAN over standard cross-entropy (CE) training on the biased sensitive attribute group distributions in Figure~\ref{fig:illustration}. Figure~\ref{fig:illustration} (a) represents a uniform/fair distribution across different sensitive attributes while a biased distribution that the {\it red} samples are more aggregated in the top left area than the {\it blue} samples are depicted in Figure~\ref{fig:illustration} (b), (c).
Then with the vanilla CE training,  Figure~\ref{fig:illustration} (a) ends up with a fair classifier determined by the ground truth task labels ({\it shapes}) while a biased classification head determined by sensitive attributes ({\it colors}) is generated in  Figure~\ref{fig:illustration} (b). Instead, our RAAN method generates a fair classifier in Figure~\ref{fig:illustration} (c), the same as a classifier learned from the Figure~\ref{fig:illustration} (a) generated from a fair distribution.
To this end, the main contributions of our work are  summarized below:
\begin{itemize}
  
    \item We propose a robust loss RAAN to debias the classification head by assigning adversarial robust weights defined on the top of biased representation space.
    When the representation is parameterized by trainable encoders such as convolutional layers in ResNets, RAAN is able to further debiase the representation distribution.
    \item We propose an efficient \underline{S}tochastic \underline{C}ompositional algorithm framework for RAAN (SCRAAN), which includes the SGD-style and Adam-style updates with theoretical guarantee. 
    \item Empirical studies on fairness-related datasets verify the supreme performance of the proposed SCRAAN on two fairness, Equalized Odd difference ($\Delta$EO), Demographic Parity difference ($\Delta$DP) and worst group accuracy.
\end{itemize}

\begin{figure}
\centering
   \includegraphics[width = 0.8\linewidth]{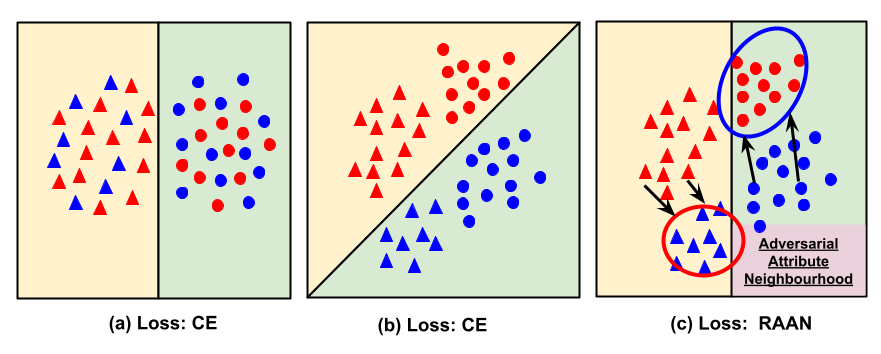}
    \caption{The influence of different protected group distributions on the classification head. The {\it colors} (\{{\it red, blue}\}) represent the sensitive attributes and {\it shapes} (\{{\it triangle, circle}\}) represent the ground truth class labels. Figures (a), (b) are optimized using vanilla CE loss, while the figure (c) is optimized using the proposed RAAN loss defined on the adversarial attributes neighborhood. The yellow and green background denote the predicted classification space. }
    \label{fig:illustration}
\end{figure}


\section{Related Work}
\noindent
\textbf{Bias Mitigation} 
To address the social bias towards certain demographic groups in deep neural network (DNN) models~\citep{lin2022quantifying, zhang2021measuring, kiritchenko2018examining, adragna2020fairness,buolamwini2018gender}, many efficient methods have been proposed to reduce the model discrimination~\citep{wang2019repairing, wadsworth2018achieving, edwards2015censoring, kim2019learning, elazar2018adversarial,singh2020don, zunino2021explainable, rieger2020interpretations, liu2019incorporating, kusner2017counterfactual, kilbertus2017avoiding, cheng2021fairfil, kang2019decoupling}.
Most methods in the above literature mainly focus on improving the fairness of encoder representation. The authors of~\citep{wang2019repairing, wadsworth2018achieving, edwards2015censoring, elazar2018adversarial} took the advantage of the adversarial training to reduce the  group discrimination. \cite{rieger2020interpretations, zunino2021explainable} made use of the model explainability to remove subset features that incurs bias, while \cite{singh2020don, kim2019learning} concentrated on the causal fairness features to get rid of undesirable bias correlation in the training.~\citet{bechavod2017penalizing} penalized unfairness by using surrogate functions of fairness metrics as regularizers.
However, directly working on a biased representation to improves classification-head remains rare. Recently, the RNF method~\citep{du2021fairness} averages the representation of sample pairs from different protected groups in the biased representation space to remove the bias in the classification head. In this paper, we propose a principled RAAN objective capable of debiasing encoder representations and classification heads at the same time.

\noindent
\textbf{Robust Loss}
Several robust loss has been proposed to improve the model robustness for different tasks. The general cross entropy (GCE) loss was proposed to solve the noisy label problem which emphasizes more on the clean samples~\citep{zhang2018generalized}. For the data imbalanced problem, distributionally robust learning (DRO)~\citep{qi2020attentional,li2020tilted, sagawa2019distributionally, qi2020simple} and class balance loss~\citep{cui2019class, cao2019learning} use instance-level and class-level robust weights to pay more attention on underrepresented groups, respectively. Recently, ~\citet{sagawa2019distributionally} shows that group DRO is able to prevent the models learning the specific spurious correlations. The above robust objective are defined on the loss space with the assistance of label information. Exploiting useful information from feature representation space to further benefit the task-specific training remains under-explored.


\noindent
\textbf{Invariant Risk Minimization (IRM)}  IRM~\citep{arjovsky2019invariant} is a novel paradigm to enhance model generalization in domain adaptation by learning the invariant sample feature representations across different "domains" or "environments". By optimizing a practical version of IRM in the toxicity classification usecase study,~\cite{adragna2020fairness} shows the strength of IRM over ERM in improving the fairness of classifiers on the biased subsets of Civil Comments dataset. To elicit an invariant feature representation, IRM is casted into a constrained (bi-level) optimization problem where the classifier $\w_c$ is constrained on a optimal uncertainty set. Instead, the RAAN objective constrains the adversarial robust weights, which are defined in pairwise representation similarity space penalized by KL divergence. When the embedding representation $\z$ is parameterized by trainable encoders $\w_f$, RAAN generates a more uniform representation space across different sensitive groups. 

\noindent
\textbf{Stochastic Optimization}
Recently, several stochastic optimization technique has been leveraged to design efficient stochastic algorithms with provable theoretical convergence for the robust surrogate objectives, such as F-measure~\citep{zhang2018faster}, average precision (AP)~\citep{qi2021stochastic}, and area under curves (AUC)~\citep{liu2019stochastic, liu2018fast, yuan2021large}. In this paper, we cast the fairness promoting RAAN loss as a two-level stochastic coupled compositional function with a general formulation of $\E_{\xi}[f(\E_{\zeta} g(\w;\zeta,\xi))]$, where $\xi,\zeta$ are independent and $\xi$ has a finite support. By exploring the advanced stochastic compositional optimization technique~\citep{wang2017stochastic, NEURIPS2021_533fa796}, a stochastic algorithm SCRANN with both SGD-style and Adam-style updates is proposed to solve the RAAN with provable convergence.

\section{Robust Adversarial Attribute Neighbourhood  (RAAN) Loss}
\label{sec:RAAN}
\subsection{Notations}
We first introduce some notations. 
The collected data is denoted by $\D =\{\d\}_{i=1}^n= \{(\x_i,y_i,a_i)\}_{i=1}^n$, where $\x_i\in \mathcal X$ is the data, $y_i\in\mathcal Y$ is the label, $a_i \in \A$ is the corresponding attribute (e.g., race, gender), and $n$ is the number of samples. We divide the data into different subsets based on labels and attributes. For any label $c\in \mathcal Y$ and attribute $a\in\mathcal A$, we denote $\D^c_a = \{(\x_i,y_i,a_i)|a_i = a \wedge y_i = c \  \}_{i=1}^n$ and 
$\D^c =  \{(\x_i,y_i,a_i)|y_i = c \ \}_{i=1}^n$. Then we have ${\D^c = \cup \{\D_a^c \}_{a=1}^{|A|}}$. Given a  deep neural network, the model weights $\w$ can be decomposed into two parts, the \underline{F}eature presentation parameters $\w_f$ and the \underline{C}lassification head parameters $\w_c$, i.e, $\w = [\w_f, \w_c]$. For example, $\w_f$ and $\w_c$ are mapped into the convolutions layers and fully connected layers in ResNets, respectively. $F_{\w_f}(\cdot)$ represents the feature encoder mapping from $\X \rightarrow \Z$, and $H_{\w_c}(\cdot)$ represents the classification head mapping from $\Z \rightarrow \Y$. Then $\z_i(\w_f) = F_{\w_f}(\x_i) \in \Z$ denotes the embedding representation of the sample $\d_i$.  $H_{\w_c}(\z_i(\w_f))$ represents the output of the classification head.

The key idea of RAAN is to assign a fairness-promoting adversarial robust weight for each sample by exploring the AAN representation structure to reduce the disparity across different sensitive attributes. The AAN of the sample $\d_i = (\x_i, y_i = c, a_i=a)$ is defined as the samples from the same class but with different attributes, i.e, \underline{$\P_i = \P^{c}_{a} = \D^{c}\backslash \D^{c}_{a}$}. For example, considering a binary protected sensitive attribute $\{female, male\}$, the AAN of a sample belonging to the $male$ protected group with class label of $c\in \Y$ is the collection of the $female$ attribute samples with the same class $c$. Then the adversarial robust weight for every sample $\d_i\sim\D$ is represented as $p_i^{\text{AAN}}$, which is an aggregation of the pairwise weights between $\d_i$ and its AAN neighbours in $\P_i$.
Next, we denote the pairwise robust weights between the sample $\d_i$ and $\d_j \in \P_i$ in the representation space as $p_{ij}^{\text{AAN}}$. 
When the context is clear, we abuse the notations by using $\p^{\text{AAN}}_i =[p^{\text{AAN}}_{i1},\cdots, p^{\text{AAN}}_{ij}, \cdots] \in\R^{|\P_i|}$ to represent the pairwise robust weights vector defined in $\P_i$, i.e, the AAN of $\d_i$.

\subsection{RAAN Objective}
\label{sec:raan_objective}
To explore the AAN representation structure and obtain the pairwise robust weights, we define the following robust constrained objective for $ \forall\d_i\sim \D$,
\begin{align}
\label{eqn:l_i_lbl}
  \ell_i^{\text{AAN}} = \sum\limits_{j\in \P_i} p^{\text{AAN}}_{ij}\ell(\w;\x_j, y_j, a_j)
\end{align}
\vspace{-0.2in}
\begin{align}
\label{eqn:nb}
\text{s.t.} \ \ \ \max\limits_{\p^{\text{AAN}}_i\in \Delta^{|\P_i|}} \sum\limits_{j\in\P_i} p^{\text{AAN}}_{ij} \z_i(\w_f)^\top\z_j(\w_f) - \tau \text{KL}\left(\p^{\text{AAN}}_i, \frac{\bf 1}{|\P_i|}\right),\ \bf 1 \in \R^{|\P_i|} 
\end{align}
where $\Delta$ is a simplex that $\sum_{j=1}^{|\P_i|}p_{ij} = 1$. The robust loss (\ref{eqn:l_i_lbl}) is a weighted average combination of the AAN loss. 
The robust constraint~(\ref{eqn:nb}) is defined in the pairwise representation similarity between the sample $i$ and its AAN neighbours penalized by the KL divergence regularizer, which has been extensively studied in \underline{d}istributionally \underline{r}obust learning \underline{o}bjective (DRO) to improve the robustness of the model in the loss space~\citep{qi2020attentional}. Here, we adopt the DRO with KL divergence constraint to the representation space to generate a uniform distribution across different sensitive attributes.

Controlled by the hyperparameter $\tau$, the close form solution of $\p_i^{\text{AAN}}$ in~(\ref{eqn:nb}) guarantees that the larger the pairwise similarity $\z_i^\top(\w_f)\z_j(\w_f)$ is, the higher the $p^{\text{AAN}}_{ij}$ will be. 
When $\tau = 0$,   the close form solution of (\ref{eqn:nb}) is 1 for the pair with the largest similarity and 0 on others.  When $\tau>0$, due to the strong convexity in terms of $\p^{\text{AAN}}_i$, the close form solution of (\ref{eqn:nb}) for each pair weight between $\d_i$ and $ \d_j\in \P_i$ is:
\begin{align}
\label{eqn:pij}
 p^{\text{AAN}}_{ij} =   \frac{\exp(\frac{\z_i^\top(\w_f)\z_j(\w_f)}{\tau})}{\sum\limits_{k\in \P_i}\exp(\frac{\z_i^\top(\w_f)\z_k(\w_f)}{\tau})}.
\end{align}
Hence the larger the $\tau$ is, the more uniform of $\p_i^{\text{AAN}}$ will be. And it is apparent to see that the robust objective generates equal weights for every pair such that $p_{ij}^{\text{AAN}} = \frac{1}{|\P_i|}$ for every $\d_j\in\P_i$  when $\tau$ approaches to the infinity in~(\ref{eqn:pij}). 
When we have a fair representation, the embeddings of different protected groups are uniform distributed in the representation space. The vanilla average loss training is good enough to have a fair classification head, which equals to RAAN with $\tau$ goes to infinity. When we have biased representations, we use a smaller $\tau$ to emphasize on the similar representations that shared invariant feature from two different protected groups to reduce the bias introduced from difference of the two protected group distributions. 

To this end, after having the close form solution for every pairwise robust weights $p_{ij}^{\text{AAN}}$ (\ref{eqn:pij}) and plugging back into $\ell_i^{\text{AAN}}$ (\ref{eqn:l_i_lbl}) for any arbitrary sample $\d_i\sim \D$, the overall RAAN objective is defined as:
\begin{align}
\label{eqn:RAAN}
\text{RAAN}(\w) := \frac{1}{C}\sum_{c=1}^C \frac{1}{A}\sum_{a=1}^{A}\frac{1}{| \D^c_a|}  \sum_{i=1}^{|\D^c_a|} \ell_i^{\text{AAN}} =\frac{1}{AC}\sum\limits_{j=1}^n p_j^{\text{AAN}}\ell(\w;\x_j, y_j, a_j),
\end{align}
where $C=|\mathcal Y|$, $A=|\mathcal A|$, $\ell_i^{\text{AAN}}=\sum\limits_{j\in \P_i} p^{\text{AAN}}_{ij}\ell(\w;\x_j, y_j, a_j)$ is defined in (\ref{eqn:l_i_lbl}), $p^{\text{AAN}}_{ij}$ is defined in (\ref{eqn:pij}), and $p_j^{\text{ANN}}= \frac{1}{|\P_j|}\sum\limits_{i\in\P_j}  \frac{\exp(\frac{\z_i^\top(\w_f)\z_j(\w_f)}{\tau})}{\sum\limits_{k\in \P_i}\exp(\frac{\z_i^\top(\w_f)\z_k(\w_f)}{\tau})}$ is obtained by we aggregating all the pairwise robust weights in $\P_j$. 
Hence, the adversarial robust weights $p_j^{\text{AAN}}$ for each sample $\d_j = (\x_j, c, a)\sim \D$, encodes the intrinsic representation neighbourhood structure between the sample and its AAN neighbors $\d_i\in \P_j$ (the numerator) and normalized by the similarity pairs from the same protected groups $\d_k\in\P_i$, i.e, $\D_a^c$ (the denominator). Due to the limitation of space, we put the second equality derivation of equation~(\ref{eqn:RAAN}) in Appendix.




\subsection{Representation Learning Robust Adversarial Attribute Neighbourhood (RL-RAAN) Loss}

AANs are defined over the encoder representation outputs $\z(\w_f)$. 
By default, the RAAN loss aims to promote a fairer classification head on a fixed bias representation distribution, i.e, $\w_f$ (recall that $\w = [\w_f, \w_c]$) is not trainable. 
By parameterizing the AANs with trainable encoder parameters, i.e, $\w_f$ is trainable, we extend the RAAN to the \underline{R}epresentation \underline{L}earnining RAAN (RL-RAAN), which is able to further debias the representation encoder. Hence, RAAN optimizes the $\w_c$ while RL-RAAN jointly optimizes $[\w_f,\w_c]$. To design efficient stochastic algorithms, RL-RAAN requires more sophisticated stochastic estimators than RAAN, which we will discuss later in Section~\ref{sec:RAAN}.
Depending on whether $\w_f$ is trainable, the red dashed arrow in Figure~\ref{fig:raan_overview} describes the optional gradients backward toward the feature representations during the training.


Here, we demonstrate the effectiveness of RL-RAAN in generating a more uniform representation distribution across different sensitive groups in Figure~\ref{fig:embedding_improve}. To achieve this, we visualize the representation distributions learned from vanilla CE (left plot) and RL-RAAN (right plot) methods using the Kernel-PCA dimensionality reduction method with radial basis function (rbf) kernels. It is clear to see that {\it white} attribute samples are more clustered in the upper left corner of CE representation projection while both the {\it white} and {\it non-white} attributes samples are  uniformly distributed in the representation projection of RL-RAAN.

\begin{figure}
\centering
\begin{minipage}[c]{0.65\textwidth}
\centering
 \includegraphics[width = 0.47\textwidth]{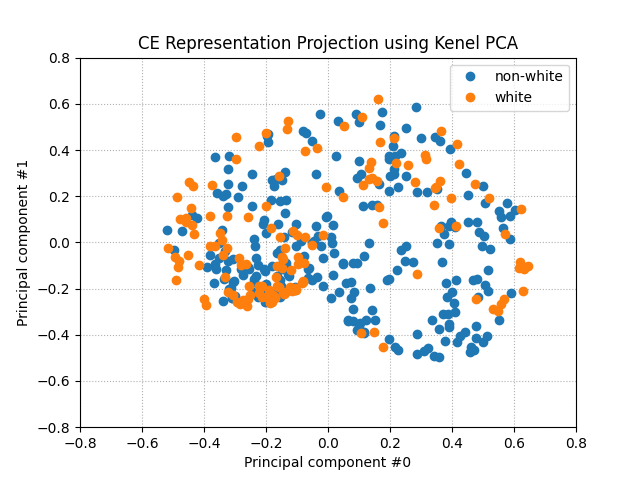}
        \includegraphics[width = 0.47\textwidth]{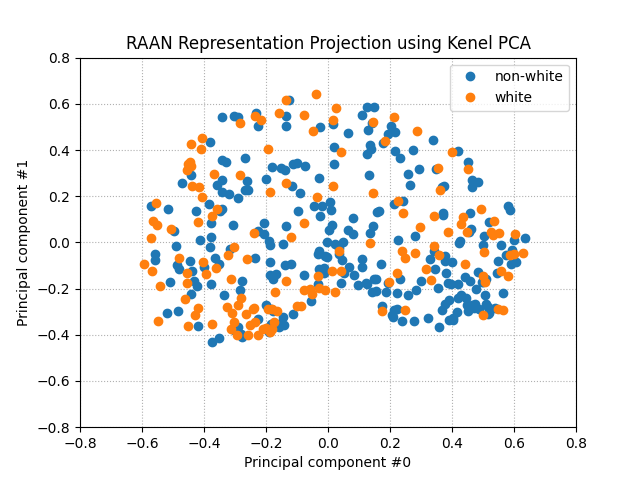}
          \captionof{figure}{Improvement of Representation Fairness}
              \label{fig:embedding_improve}
\end{minipage}
\hfill
\begin{minipage}[c]{0.33\textwidth}
\centering
    \includegraphics[width = 0.9\linewidth]{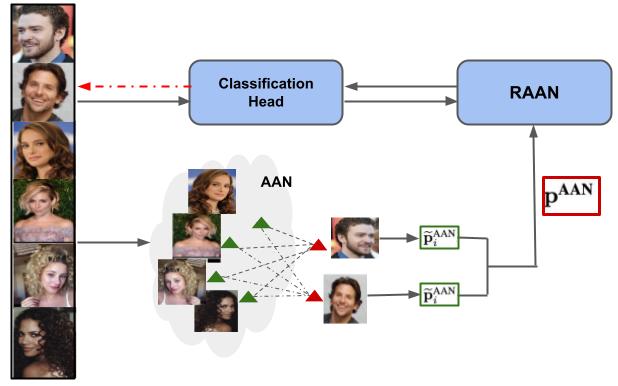}
    \caption{Training overview of (RL)-RAAN}
    \label{fig:raan_overview}
\end{minipage}
\end{figure}

\section{\underline{S}tochastic \underline{C}ompositional Optimization for RAAN (SCRAAN)}
\label{sec:SCRAAN}

In this section, we provide a general \underline{S}tochastic \underline{C}ompositional optimization algorithm framework for RAAN (SCRAAN).
The SCRAAN applies to both RAAN and RL-RAAN objective.
We first show that (RL)-RAAN is a {\bf two-level stochastic coupled compositional function} and then design stochastic algorithms under the framework of stochastic gradient descent updates, SGD and Adam~\citep{kingma2014adam},  with theorectical guarantee in Algorithm~\ref{alg:SCRAAN}.

Let $\I(c)$ denotes the indicator function that equals to 1 when $c$ is true and equals to 0 otherwise.  $g(\w; \x_i,\x_j) = [g_1(\w;\x_i,\x_j), g_2(\w; \x_i,\x_j)]^\top = \frac{n}{AC}[\exp(\frac{\z_i(\w_f)^\top\z_j(\w_f)}{\tau})\ell_j(\w;\x_j,c_j,a_j)\I(\x_j\in\P_i),$  $\exp(\frac{\z_i(\w_f)^\top\z_j(\w_f)}{\tau})\I(\x_j\in\P_i)]^\top:\R^d\rightarrow\R^2$, $g_{\x_i}(\w) =\E[g(\w;\x_i)] = \E_{\x_j\in\P_i}[g(\w;\x_i,\x_j)]$, $f(g) = \frac{g_1}{g_2}, g = [g_1, g_2]$ and $f(s) = \frac{s_1}{s_2}:\R^2 \rightarrow \R$. Then the (RL)-RAAN objectiv~(\ref{eqn:RAAN}) can be written as
\begin{align}
\label{eqn:R}
    \text{R}(\w) = \frac{1}{n}\sum\limits_{\x_i\in \D} f(g_{\x_i}(\w)) = \E_{\x_i\in \D}[f(g_{\x_i}(\w))]
\end{align}
where $g$ denotes the inner objective and $f$ denotes the outer objective. The equivalence between~(\ref{eqn:RAAN}) and~(\ref{eqn:R}) is shown in Appendix~\ref{sec:deriv_RAAN}. In the following, we use $\widetilde{\w}$ to unify the trainable parameter notation of RAAN and RL-RAAN.
Recall that $\w = [\w_f,\w_c]$, let $\widetilde{\w} = \w_c$ when $R(\w)$ represents the RAAN, and $\widetilde{\w} = \w$ when $R(\w)$ represents RL-RAAN. Then according to the chain rule, the gradient of $R(\w)$ is

\begin{minipage}[t]{0.5\textwidth}
\begin{algorithm}[H]
    \centering
    \caption{SCRAAN}\label{alg:SCRAAN}
    \begin{algorithmic}[1]
        \STATE \textbf{Input:}  Initialize $\w^1 = [\w^1_f, \w^1_c]$.
        \WHILE{first stage}
        \STATE  Train the whole model $\w$ with standard CE loss 
        \ENDWHILE
 \WHILE{second stage}
             \FOR {$t = 1,\ldots, T - 1$}
       \STATE Draw a batch samples $\{(\x_i, y_i, a_i)\}_{i=1}^B$
       \STATE $\u = \text{UG}(\B,\u,\w_t,\gamma, u_0)$
       \STATE Compute (biased) Stochastic Gradient Estimator $G(\w_t)$ by Equation~(\ref{eqn:update_g})
       \STATE Update $\w_{t+1}$ with a SGD-style method or by a Adam-style method
       \begin{align*}
           \w_{t+1} = \text{UW}(\w_t, G(\w_t))
       \end{align*}
              \ENDFOR
        \ENDWHILE
        \STATE \textbf{Return:}  $\w_R$,  $R$ is a index sampled from $ 1 \cdots T$.
    \end{algorithmic}
\end{algorithm}
\end{minipage}
\hfill
\begin{minipage}[t]{0.5\textwidth}
\begin{algorithm}[H]
    \caption{UG($\mathcal B, \u,  \w_t, \gamma, u_0$)}\label{alg:2}
    \begin{algorithmic}[1]
       \FOR{each $\x_i\in\mathcal B$}
       \STATE Construct $\widehat{\P}_i = \P_i\cap\B$ and compute $[\widehat{g}_{\x_i}(\w_t)]_1, [\widehat{g}_{\x_i}(\w_t)]_2$ by Equation~(\ref{eqn:sg})
       \STATE Compute $\u^1_{\x_i} = (1-\gamma)\u^1_{\x_i} + \gamma [\widehat{g}_{\x_i}(\w_t)]_1 $\\
        $\u^2_{\x_i} = \max((1-\gamma)\u^2_{\x_i} + \gamma [\widehat{g}_{\x_i}(\w_t)]_2, u_0)$
        \ENDFOR
        \STATE {\bf Return} $\u$
    \end{algorithmic}
\end{algorithm}
\vspace{-0.08in}
\begin{algorithm}[H]
    \caption{UW($\w_t, G(\w_t)$)}\label{alg:3}
    \begin{algorithmic}[1]
       \STATE \text{Option 1}: SGD-style update (paras: $\alpha$)
         $\w_{t+1}= \w_t - \alpha G(\w_t)$
         
       \STATE \text{Option 2}: Adam-style update (paras: $\alpha, \epsilon, \eta_1, \eta_2$)
$h_{t+1} = \eta_1 h_{t} + (1-\eta_1)G(\w_t)$ \\
$v_{t+1} = \eta_2 \hat{v}_t + (1-\eta_2)(G(\w_t))^2$ \\
$\w_{t+1} = \w_t - \alpha \frac{h_{t+1}}{\sqrt{\epsilon + \hat{v}_{t+1}}}$

where $\hat v_t = v_t$ (Adam) or $\hat v_{t}=\max(\hat v_{t-1}, v_t)$ (AMSGrad)
\STATE {\bf Return:} $\w_{t+1}$
    \end{algorithmic}
\end{algorithm}
\end{minipage}

\begin{align*}
\resizebox{0.9\linewidth}{!}{$
         \nabla_{\widetilde{\w}} R(\w)= \frac{1}{n}\sum\limits_{\x_i\in \D}  \nabla_{\widetilde{\w}} g_{\x_i}(\w)^\top\nabla f(g_{\x_i}(\w)) =\frac{1}{n} \sum\limits_{\x_i\in \D} ( [\nabla_{\widetilde{\w}} g_{\x_i}(\w)]_1^\top, [\nabla_{\widetilde{\w}} g_{\x_i}(\w)]_2^\top )\left(\begin{array}{c} \frac{1}{[g_{\x_i}(\w)]_2} \\
   -\frac{[g_{\x_i}(\w)]_1}{[g_{\x_i}(\w)]_2^2}
    \end{array}\right)
    $}
\end{align*}
To provide efficient stochastic optimizations, we approximate the gradients of $\nabla R(\w)$ with the stochastic estimators in Algorithm~\ref{alg:SCRAAN}.  
Let $\B$ denotes a $B$ sample set randomly generated from $\D$. For each sample $\x_i\in \B$, we approximate the $\nabla_{\widetilde{\w}} g_{\x_i}(\w)$ using the stochastic gradient on the current batch, $\nabla_{\widetilde{\w}}\widehat{g}_{\x_i}(\w)$, i.e, $([\nabla_{\widetilde{\w}} \widehat{g}_{\x_i}(\w)]_1^\top, [\nabla_{\widetilde{\w}} \widehat{g}_{\x_i}(\w)]_2^\top )$.
Denotes the stochastic AAN samples in current batch $\B$ for the sample $i$ is denoted as $\widehat{\P}_i = \P_i\cap\B$ and $\exp^\tau_{ij}(\w^t_f) =  \exp(\frac{\z_i(\w^t_f)^\top\z_j(\w^t_f)}{\tau})$, then stochastic estimators for (RL)-RAAN are represented as:
\begin{equation}
\label{eqn:raan_nabla_g}
\resizebox{0.93\textwidth}{!}{$
{[}\nabla_{\widetilde{\w}} \widehat{g}_{\x_i}(\w_t)]_1^\top = \left \{
  \begin{aligned}   &\frac{n}{AC}\frac{1}{|\hat{\P}_i|}\sum\limits_{\x_j\in \hat{\P}_i} \exp^\tau_{ij}(\w^t_f) \nabla_{\w_c} \ell_j(\w_t)^\top  &&  \  \text{RAAN} \\
&\frac{n}{AC} \frac{1}{|\hat{\P}_i|}\sum\limits_{\x_j\in \hat{\P}_i}\exp^\tau_{ij}(\w^t_f)(\nabla_{\w} \ell_j(\w_t)^\top+ (\z_i(\w^t_f) + \z_j(\w^t_f))^\top \ell_j(\w_t)) && \text{RL-RAAN} 
\end{aligned}  
 \right.
 $}
\end{equation}
${[}\nabla_{\widetilde{\w}} \widehat{g}_{\x_i}(\w)]_2^\top $ is a $\textbf{0}$ vector  for RAAN and the dimension of $\textbf{0}$ equals to the dimension of $\w_c$. For RL-RAAN,  equals to
${[}\nabla_{\widetilde{\w}} \widehat{g}_{\x_i}(\w)]_2^\top  = \frac{n}{AC|\hat{\P}_i|}\sum_{\x_j\in \hat{\P}_i} \exp^\tau_{ij}(\w_f)(\z_i(\w_f^t) + \z_j(\w_f^t))$.

 To estimate $g_{\x_i}(\w)$, however, the stochastic objective $\widehat{g}_{\x_i}(\w)$ is not enough to control the approximation error such  that the convergence of Algorithm~\ref{alg:SCRAAN} can be guaranteed. We borrow a technique from the stochastic compositional optimization literature~\citep{wang2017stochastic}
by using a moving average estimator to estimate $g_{\x_i}(\w)$ for all samples. We maintain a matrix $\u =[\u_1, \u_2]$ and each of a column is indexed by a sample $\x_i\sim\D$ corresponding to the the moving average stochastic estimator of $g_{\x_i}(\w)$. 
The Step 3 of Algorithm~\ref{alg:2} describes the updates of $\u$, in which $u_0$ is a small constant to address the numeric issue that does not influence the convergence analysis and the stochastic estimator $\widehat{g}_{\x_i}(\w)$ for sample $i$ is
\begin{equation}
\label{eqn:sg}
   \resizebox{0.93\textwidth}{!}{
 $\begin{aligned}
 \text{[}\widehat{g}_{\x_i}(\w_t)]_1 =
      \frac{n}{AC} \frac{1}{|\widehat{\P}_i|} \sum\limits_{j\in\hat{\P}_i}\exp_{ij}^\tau(\w_f^t)\ell_j(\w_t),\hfill
     [\widehat{g}_{\x_i}(\w_t)]_2  = 
      \frac{n}{AC}\frac{1}{|\widehat{\P}_i|} \sum\limits_{j\in\widehat{\P}_i}\exp_{ij}^\tau(\w_f^t)
\end{aligned}$
}
\end{equation}

To sum up, the overall stochastic estimator $G(\w)$ for $\nabla R(\w)$ in a batch
where the stochastic inner objective gradient estimator for (RL)-RAAN:
\begin{align}
\label{eqn:update_g}
    G(\w)  = \frac{1}{B}\sum\limits_{i=1}^B \nabla_{\widetilde{\w}} \hat{g}_{\x_i}(\w)^\top\left(\begin{array}{c} \frac{1}{\u^2_{\x_i}} \\
   -\frac{\u^1_{\x_i}}{[\u^2_{\x_i}]^2} )
    \end{array}\right) 
\end{align}
Finally, we apply both the SGD-style and Adam-style updates for $\w$ in Algorithm~\ref{alg:3} .
Next we provide the theoretical analysis for SCRAAN.
\begin{thm}\label{thm:main-Adam}
Suppose Assumption ~\ref{ass:1} holds, $\forall\  t\in 1,\cdots, T$, and $T>n$,  let the parameters be 
1) $\alpha = \frac{1}{n^{2/5}T^{3/5}}$,$\gamma = \frac{n^{2/5}}{T^{2/5}} $ for the SGD updates; 2) $\eta_1\leq \sqrt{\eta_2}\leq 1$, $\alpha = \frac{1}{n^{2/5}T^{3/5}}$,$\gamma = \frac{n^{2/5}}{T^{2/5}} $ for the AMSGrad updates. Then after running $T$ iterations, SCRAAN with SGD-style updates or Adam-style update satisfies
$$ \E\left[ \frac{1}{T}\sum\limits_{t=1}^T\|\nabla R(\w_t)\|^2\right] \leq  O\left(\frac{n^{2/5}}{T^{2/5}}
\right),$$
where  $O$ suppresses constant numbers.
\end{thm}
\noindent
\textbf{Remark:} Even though RAAN and RL-RAAN enjoys the same iteration complexity in Theorem~\ref{thm:main-Adam}, the stochastic estimator  ${[}\nabla_{\widetilde{\w}} \widehat{g}_{\x_i}(\w)]_2^\top$ is $\textbf{0}$ leads to a simpler optimization for RAAN such that we only need to maintain and update $\u^2_{\x_i}$ to calculate $G(\w_t) = 1/|\B|\sum_{i=1}^{|\B|}  [\nabla_{\w_c} g_{\x_i}(\w)]_1/\u_{\x_i}^2$.
There are other more sophisticated optimizers, such as SOX~\citep{pmlr-v162-wang22ak}, MOAP~\citep{wang2022momentum} and BSGD~\citep{hu2020biased}, can also be applied to solve~(\ref{eqn:RAAN}), which we leave as a future exploration direction. The derivation of Theorem~\ref{thm:main-Adam} is provided in Appendix~\ref{sec:thm_drivation}

\section{Empirical Studies}
\label{sec:empirical_studies_sec}
In this section, we conduct empirical studies on fourdatasets: Adult~\citep{kohavi1996scaling}, Medical Expenditure (MEPS)\citep{cohen2003design},   CelebA~\citep{liu2015deep}, and Civil Comments\footnote{\url{https://www.tensorflow.org/datasets/catalog/civil_ comments}} dataset in the NLP. We compare the proposed methods with: 1) bias mitigation methods: RNF~\citep{du2021fairness}, Adversarial learning~\citep{zhang2018mitigating}, regularization method~\citep{bechavod2017penalizing}. 2) robust optimization methods: Empirical Risk Minimization (ERM), Group Distributionally Robust Learning (GroupDRO)~\citep{sagawa2019distributionally} and Invariant Risk Minimization (IRM)\citep{adragna2020fairness, arjovsky2019invariant}.

\noindent
\textbf{Datasets: } For the two benchmark tabular datasets, the Adult dataset used to predict whether a person's annual income higher than 50K while the goal of MEPS is to predict whether a patient could have a high utilization. For the CelebA image dataset, we want to predict whether a person has wavy hair or not. Civil Comments dataset is an NLP dataset aims to predict the binary toxicity label for online comments. Accordingly, the protected sensitive attribute is \textbf{gender} $\in\{female, male\}$ \footnote{For the gender attribute, there are more than binary attributes. For example, it contains but not limited to female, male and transgender are included to name a few. Here, due to the limited size of the datasets, we only consider female and male attributes in this paper.}on the Adult and CelebA datasets, and the protected sensitive attribute is \textbf{race} $\in \{white, nonwhite\}$ on MEPS. We consider four different types of demographic sensitive attributes for each comments belonging to \{{\bf Black, Muslim, LGBTQ, NeuroDiverse}\} on the Civil Comments dataset.
The training data size varies from 11362 in MEPS, 33120 in Adult to 194599 in CelebA. Civil Comments Dataset contains 2 million online news articles comments that are annotated by toxicity. Following the setting in~\citep{adragna2020fairness},  subsets of 450,000 comments for each attribute are constructed.

\noindent
\textbf{Metrics:}
In the experiments, we compare two fairness metric equalized odd difference ($\Delta$EO), demographic parity difference ($\Delta$DP) between different methods given the same accuracy and worst group accuracy in terms of \{{\it Label $\times$ Attribute}\}. $\Delta$DP measures the difference in probability of favorable outcomes between unprivileged and privileged groups $\Delta\text{DP} = |\text{PR}_0 - \text{PR}_1|$, where $\text{PR}_0 = p(\hat{y} = 1|a= 0)$, and $\text{PR}_1 = p(\hat{y} = 1|a= 1)$. $\Delta$EO requires favorable outcomes to be independent of the protected class attribute $a$, conditioned on the ground truth label $\y$. $\Delta\text{EO} = |\text{TPR}_0 - \text{TPR}_1| + |\text{FPR}_0 - \text{FPR}_1|$, where $\text{TPR}_0 = p(\hat{y}=1|a = 0, y =1)$, $\text{TPR}_1 = p(\hat{y}=1|a = 1, y =1)$, $\text{FPR}_0 = p(\hat{y}=1|a = 0, y =0)$, and $\text{FPR}_1 = p(\hat{y}=1|a = 1, y =0)$, $|\cdot|$ denotes absolutew value.

\subsection{Comparison with Bias Mitigation Methods}

\textbf{Baselines} \textbf{Vanilla} refers to the standard CE training with cross entropy loss, \textbf{RNF} represents the representation neutralization method in~\citep{du2021fairness},  \textbf{Adversarial} denotes the adversarial training method \citep{zhang2018mitigating} that mitigates biases by simultaneously learning a predictor and an adversary, \textbf{(RL-)EOR}~\citep{bechavod2017penalizing} is a regularization method that uses a surrogate  function of $\Delta$EO as the regularizer.
The comparison of baselines are described in Table~\ref{tab:difference}.
We report two version of experiments for the regularization method and our proposed method for debiasing the representation encoder and debiasing the classification head, i.e, RL-EOR and EOR,  RL-RAAN and RAAN, respectively.

\begin{wraptable}{r}{6cm}
\centering
 \caption{Comparison of baseline methods}
\label{tab:difference}
\resizebox{\linewidth}{!}{  
\begin{tabular}{c|c|c} 
\toprule
            &  \makecell[c]{Debiasing \\ Representation Encoder} &  \makecell[c]{Debiasing \\Classification Head} \\ \hline
Vanilla     &                $\times$                  &            $\times$                   \\
Adversarial &       $ \checkmark $                         &    $\times$                          \\
(RL-)EOR         &         $\checkmark$                          &    $  \checkmark$                          \\
RNF  & $\times$ & $\checkmark$\\
(RL-)RAAN        &       $\checkmark$                            &   $\checkmark$  \\               \bottomrule            
\end{tabular}
}
\end{wraptable} 

\noindent
\textbf{Models and Parameter settings}
Following the experimental setting in~\citep{du2021fairness}, we train a three layer MLP for Adult and MEPS datasets and ResNet-18 for CelebA. The details of the MLP networks are provided in the Appendix~\ref{sec:mlp_structure}. We adopt the two stage bias mitigation training scheme such that we apply the vanilla CE method in the first stage and then debias the representation encoder $\w_f$ or classification head $\w_c$ in the second stage. The representation encoder is fixed when we debias the classification head $\w_c$.
We train 10 epochs per stage for MEPS and Adult datasets, and 5 epochs per stage for the CelebA. We report the final $\Delta$EO, $\Delta$DP and worst group accuracy in terms of \{$ Label \times Attribute$\} on the test data at the end of the training. The batch size of MEPS and Adult is 64 by default, and the batch size of celebA is 190. We use the Adam-style/SGD-style SCRAAN optimize RAAN, and Adam/SGD optimizer for other baselines.
For all the methods, the learning rate $\alpha$ is tuned in $\{1e$-$2$, $1e$-$3$, $1e$-$4\}$ and $\tau\in\{0.1:0.2:2\}$.
For the RAAN, we tune $\gamma \in \{0.1, 0.5,0.9\}$. The regularizer hyperparameter of RNF $\alpha'$ is tuned in $\{0,\ 1e$-$4,\ 2e$-$4,\ 3e$-$4,\ 4e$-$4\}$. And the regularizer parameter in Adversarial and EOR is tuned $\in\{0.01:0.02:0.1\}$. The learning rate for the adversarial head $\alpha_\text{Adv}$ in Adversarial is tuned in $\{1e$-$2$, $1e$-$3$, $1e$-$4\}$.



\begin{figure}[t]
    \centering
    \includegraphics[width = 0.32\linewidth]{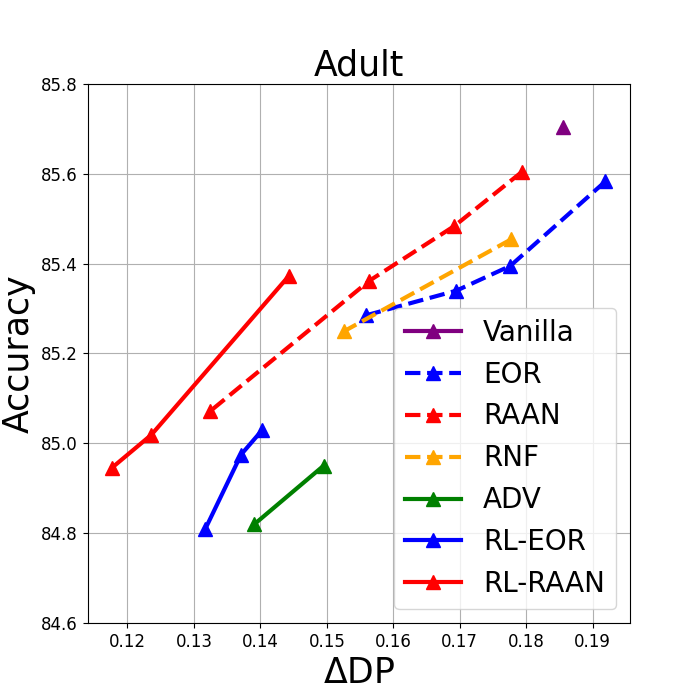}
       \includegraphics[width = 0.32\linewidth]{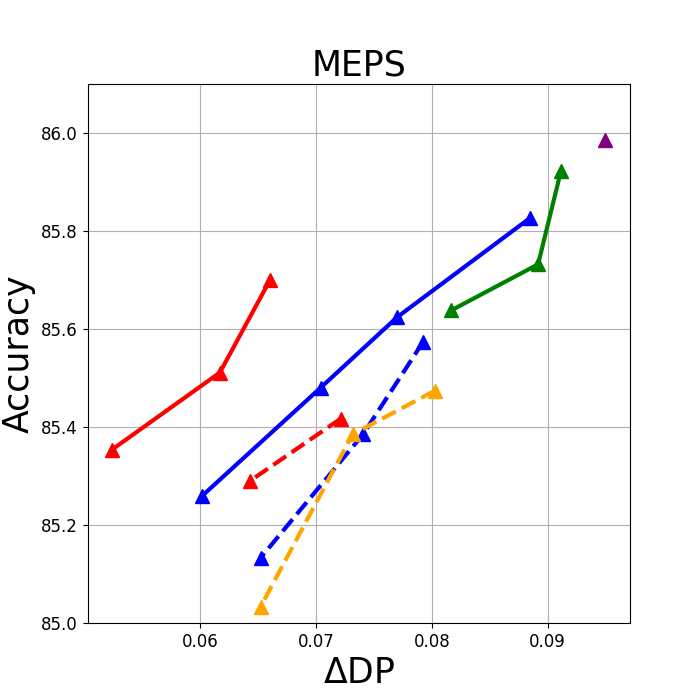}
           \includegraphics[width = 0.32 \linewidth]{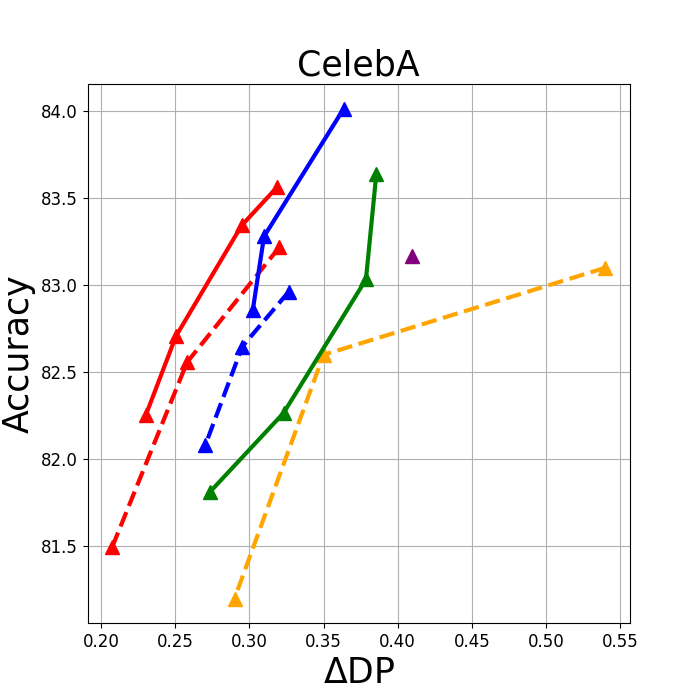}
           \\
    \includegraphics[width = 0.32\linewidth]{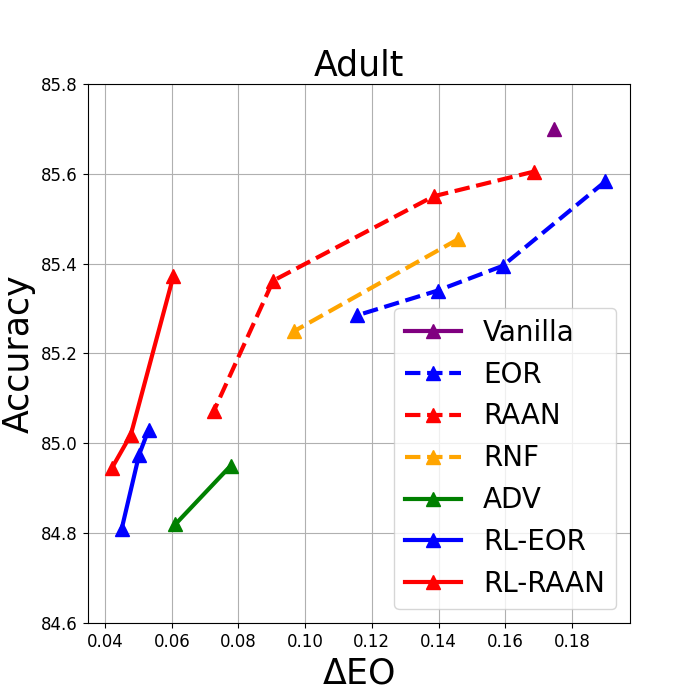}
    \includegraphics[width = 0.32\linewidth]{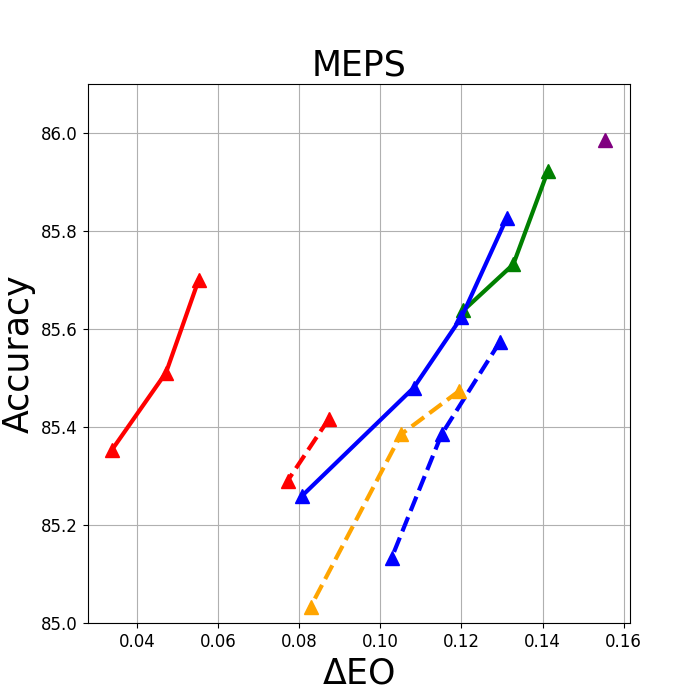} 
    \includegraphics[width = 0.32 \linewidth]{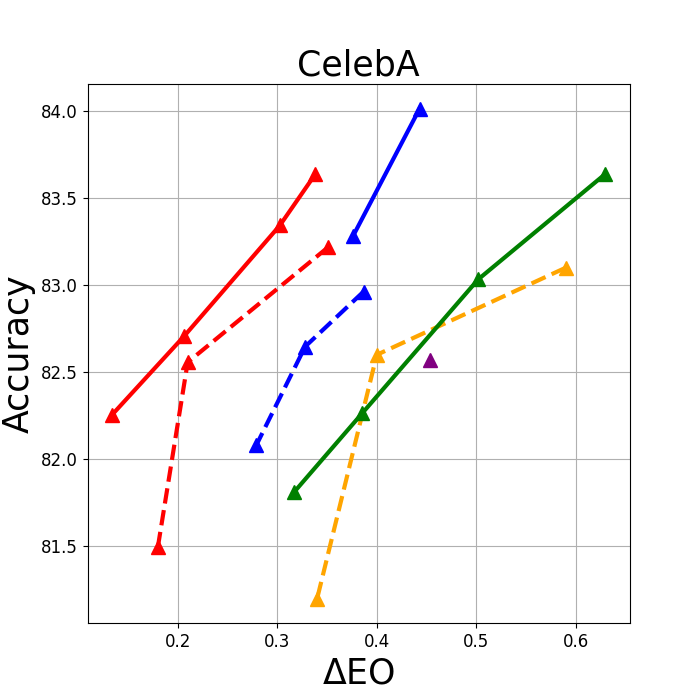}
    \caption{ $\Delta$DP and $\Delta$EO experimental results of different methods optimized by Adam-style SCRAAN on Adult, MEPS and CELEBA, respectively. The results are reported over 5 runs.}
 \label{fig:sota_res_adam}
\end{figure}

\begin{table}[htbp]
\caption{Worst group accuracy over 5 independent runs.}
\label{tab:worst_group_acc}
\centering
\resizebox{0.9\linewidth}{!}{
\begin{tabular}{c|c|c|c|c|c|c}\toprule
            & \multicolumn{2}{c|}{MEPS} & \multicolumn{2}{c|}{Adult} & \multicolumn{2}{c}{CelebA} \\ 
Optimizer   & SGD-style  & Adam-style  & SGD-style   & Adam-style  & SGD-style   & Adam-style   \\ \hline

Vanilla     &      27.61  $\pm$ 0.12    &   32.80   $\pm$ 0.15        &    55.49 $\pm$ 0.11         &   53.96 $\pm$ 0.14           &     41.32 $\pm$ 0.92       &       40.65   $\pm$ 0.85    \\ 
\hline

EOR         &      27.82 $\pm$ 0.15       &  30.23 $\pm$  $0.15$          &   51.01 $\pm$ 0.51         &        52.74    $\pm$ 0.21  &     46.02 $\pm$ 0.97        &          45.31 $\pm$ 1.58    \\
RNF         &      -      &     31.91  $\pm$ 0.23     &      -     &       58.85  $\pm$ 0.41    &    -          &       50.23  $\pm$ 0.69        \\
RAAN        &  {\bf 28.01}  $\pm$ 0.17       & {\bf 33.10} $\pm$ 0.10             &     {\bf 58.23} $\pm$ 0.13    &        
{\bf 59.76} $\pm$ 0.31   &     {\bf 53.47}  $\pm$ 1.24    &      {\bf 55.81} $\pm$ 1.12          \\ 
\hline
Adversarial &      28.31 $\pm$ 0.14      &     32.76 $\pm$ 0.13        &    54.72 $\pm$ 0.76      &     55.49 $\pm 0.31$      &   41.51 $\pm$ 0.98          &        40.15  $\pm$ 0.69    \\
RL-EOR    &    29.52 $\pm$ 0.21         &   32.00 $\pm$ 0.23         &        59.15 $\pm$ 0.11   &           57.41 $\pm$ 0.41  &           44.61 $\pm$ 0.83   &         46.89 $\pm$ 1.51    \\ 
RL-RAAN     & {\bf 30.00} $\pm$ 0.45           &  {\bf 35.21} $\pm$ 0.33           &   {\bf 65.94} $\pm$ 0.86       &    {\bf 68.10} $\pm$ 0.39       &      {\bf 58.61} $\pm$ 1.01      & {\bf 66.44}   $\pm$ 0.72     \\ \bottomrule
\end{tabular}}
\end{table}

\noindent
\textbf{Experimental results} 
In this section, we report the comparison experimental results between (RL)-RAAN optimized by Adam-style SCRAAN and other baselines optimized by Adam~\citep{kingma2014adam}.  The experimental results of the SGD-style optimizers are presented in Appendix~\ref{sec:SGD-style}.
We present the Accuracy vs $\Delta$DP, Accuracy vs $\Delta$EO results for different methods in Figure~\ref{fig:sota_res_adam}. And the worst group accuracy are reported in Table~\ref{tab:worst_group_acc}.
For the same accuracy, the smaller the value of $\Delta$DP and $\Delta$EO is, the better the method will be. It is worth to notice that we can not have a valid result for when replicating RNF method using SGD optimizer. By balancing between accuracy and $\Delta$DP, $\Delta$EO, RL-RAAN has the best results on all datasets.
For the RAAN method, it has smallest $\Delta$DP and $\Delta$EO among the methods that debias on the classification head (the dashed lines). 
Besides the Vanilla  method, EOR has the worst $\Delta$DP and $\Delta$EO in Adult, while Adversarial method has the worst $\Delta$DP and $\Delta$EO in MEPS and CelebA. When it comes to the worst group accuracy, RAAN and RL-RAAN achieve the best performance in debiasing the classification and representation encoder, respectively. In addition we provide ablation studies in the Apendx~\ref{sec:aba_study_scraan}.

 \subsection{Comparison with Stochastic Optimization Methods}
In this section, we compare RAAN with stochastic optimization methods including ERM, Group DRO, and IRM on the subsets of Civil Comments dataset. IRM and Group DRO have been proved to prevent models from learning prespecified spurious correlations~\citep{adragna2020fairness, sagawa2019distributionally}.
 We consider three different enviroments for training and testing~\citep{adragna2020fairness}.
 We set the sample size to be the same for the three environments.
 For each environment, we have a balanced number of comments for each class and each attribute, i.e, half are non-toxic ($y = 0$) and half are toxic comments ($y=0$). Similarly, for each sensitive demographic attribute in $\{\text{Black, Muslim, LGBTQ, NeuroDiverse}\}$, half of the comments are about the sensitive demographic attribute $(a=1)$ and half are not $(a=0)$. 
 We define the label switching probability $p_e =p(a=z|y=1-z), \forall z\in\{0, 1\}$ to introduce the spurious correlations between the sensitive attributes and class labels and quantify the difference between different environments.
The training datasets include two enviorments with $p_e = 0.1$ and $0.2$, while $p_e = 0.9$ in testing data environment. 

\noindent
\textbf{Baselines} For the ERM, we optimizes vanilla CE using Adam optimizer. IRM optimize a practical variant objective for the linear invariant predictor, i.e, Equation (IRMv1), proposed in~\citep{arjovsky2019invariant} using Adam optimizer. Group DRO~\citep{sagawa2019distributionally} aims to minimize the worst group accuracy. For proposed methods, we optimize the RAAN using the Adam-style SCRAAN. SCRAAN and Group DRO explicitly make use of the sensitive attributes information to construct AAN and calculate group loss, respectively.

\noindent
\textbf{Model and Parameter Settings} We train a logistic regression with l2 regularization as the toxicity classification model~\citep{adragna2020fairness} by converting each comment into a sentence embedding representing its semantic content using a pre-trained Sentence-BERT model~\citep{reimers2019sentence}.
All the learning rates are finetuned
using grid search between $\{0.0001, 0.01\}$. The hyper parameter of RAAN follows previous section. For the Group DRO the temperature parameter $\eta$ is tuned in $\{1:0.2:2\}$. The hyperparemeter for IRM are tuned following~\citep{adragna2020fairness}.

\noindent
\textbf{Experimental Results}
The experimental results are reported in Table~\ref{tab:irm}. We can see that SCRAAN and Group DRO have a significant improvement over ERM and IRM on all three evaluation metric, which implies the effectiveness of sensitive attributes information to reduce model bias. When compared with Group DRO the SCRAAN and Group DRO, SCRAAN has comparable results in terms of group accuracy while performs better on  $\Delta$EO. This makes sense as the objective of Group DRO aims to minimize the worst group loss, while RAAN focuses on improving the fairness of different groups.

 \begin{table}[htbp]
\centering
\caption{Experimental results on the testing environments over 5 independent runs.}
\label{tab:irm}
 \resizebox{\linewidth}{!}{  
\begin{tabular}{c|cccc|cccc|cccc} \toprule
           & \multicolumn{4}{c|}{Accuracy} & \multicolumn{4}{c|}{Worst Group Accuracy} & \multicolumn{4}{c}{$\Delta$EO} \\
Sens Att   & ERM     & IRM   

& Group DRO & SCRAAN   & ERM        & IRM  & Group DRO      & SCRAAN        & ERM    & IRM  & Group DRO & SCRAAN \\ \hline
 Black    & 47.04 $\pm$ 0.9  & 55.31 $\pm$ 1.2 & 67.32 $\pm$ 1.0 & {\bf 71.29} $\pm$ 1.0    &        35.01 $\pm$ 0.7 &   45.01$\pm$ 0.9      & {\bf 64.91}$\pm$ 1.0 & 64.23 $\pm$ 0.9       & 52.23 $\pm$ 3.4   & 30.90 $\pm$ 4.1 & 12.82 $\pm$ 2.7 & {\bf 4.77} $\pm$ 2.1  \\ 
  Muslim    &    49.23 $\pm$ 0.9    &     59.08 $\pm$ 1.4   & 66.37$\pm$ 1.2 &  \textbf{71.82} $\pm$ 1.0       &   36.92    $\pm$ 0.8    &    55.92    $\pm$ 1.7             &  62.45 $\pm$ 1.1  & {\bf 62.51} $\pm$ 0.9           &    47.44 $\pm$ 2.1    &    25.93 $\pm$ 4.1   & 11.79 $\pm$ 2.2 &    {\bf 7.47 }  $\pm$ 1.9\\
NeuroDiv   & 65.18 $\pm$ 1.0  & 63.60 $\pm$ 1.3   & 68.03 $\pm$ 1.1 & {\bf 68.17 }$\pm$ 1.2   &      56.26  $\pm$ 1.0    &    45.75   $\pm$ 0.8    & 63.76 $\pm$0.9 & {\bf 64.06} $\pm$ 1.1        & 26.53 $\pm$ 1.7  & 26.26 $\pm$ 2.1 & 10.75 $\pm$ 1.6 & {\bf 4.94} $\pm$ 0.9  \\
LGBTQ & 56.67  $\pm$ 1.1 & 61.99 $\pm$ 1.5   & 66.76 $\pm$ 1.3 & {\bf 69.54 } $\pm$ 1.1   &       42.58  $\pm$ 0.7   &        52.74  $\pm$ 1.2  & 63.10 $\pm$ 0.9  &   \textbf{67.31} $\pm$ 1.0       & 37.53 $\pm$ 2.1  & 25.00  $\pm$ 1.9 &18.56 $\pm$ 2.1 & {\bf 7.65} $\pm$ 1.3   \\
\bottomrule
\end{tabular}
}
\end{table}

\section{Conclusion}
In this paper, we propose a robust loss RAAN that is able to reduce the bias of the classification head and improve the fairness of representation encoder. Then an optimization framework SCRAAN has been developed for handling RAAN with provable theoretical convergence guarantee. Comprehensive studies on several fairness-related benchmark datasets verify the effectiveness of the proposed methods.

\bibliography{NARL.bib}

\newpage

\section*{Appendix}

\subsection{MLP Network Structures}
\label{sec:mlp_structure}
To gain the feature representations, we use a three layer MLP for both Adult and MEPS datasets. The input and hidden layers are following up with a ReLU activation layer and a 0.2 drop out layer, respectively.
The input size is 120 for Adult dataset and 138 for the MEPS dataset. The hidden size is 50 for both datasets. 
After that, we use a two layer classification head with a ReLU and 0.2 drop out layer for the second stage training prediction.

\subsection{Ablation Stuides of SCRAAN}
\label{sec:aba_study_scraan}
$\gamma$ and $\tau$  are two key paramters for SCRAAN. $\tau$ is the key hyperparameter  to control the pairwise robust weights aggregation for RAAN. $\gamma$ is designed for the stability and theoretical guarantees of Algorithm~\ref{alg:SCRAAN}. We provide ablation studies for the two parameters independently. 

To analyze the robustness of Algorithm~\ref{alg:SCRAAN} in terms of $\gamma$,
we report $\Delta$EO, $\Delta$DP given the accuracy 85.3 for the Adam-style SCRAAN and 84.95 for the SGD-style SCRAAN on Adult dataset in Figure~\ref{fig:gamma_robust} by varing $\gamma\in\{0.1:0.1:0.9\}$ and fixing $\tau = 0.9$. It is obvious to see that both SGD-style SCRAAN and Adam-style SCRAAN are robust enough to have valid fairness evaluations.

Similarly, for the parameter $\tau$, we report $\Delta$EO, $\Delta$DP of Adam-style SCRAAN to achieve accuracy 85.3 on Adult dataset by varing $\tau = \{0.1:0.2:1.9\}$ with $\gamma = 0.5$. We can see that by hypertuning $\tau$ in a reason range, we are able to find a $\tau$ achieves lowest $\Delta$EO and  $\Delta$DP at the same time.

 \begin{figure}[h]
\centering
\begin{minipage}[c]{0.64\textwidth}
\centering
 \includegraphics[width = 0.48\linewidth]{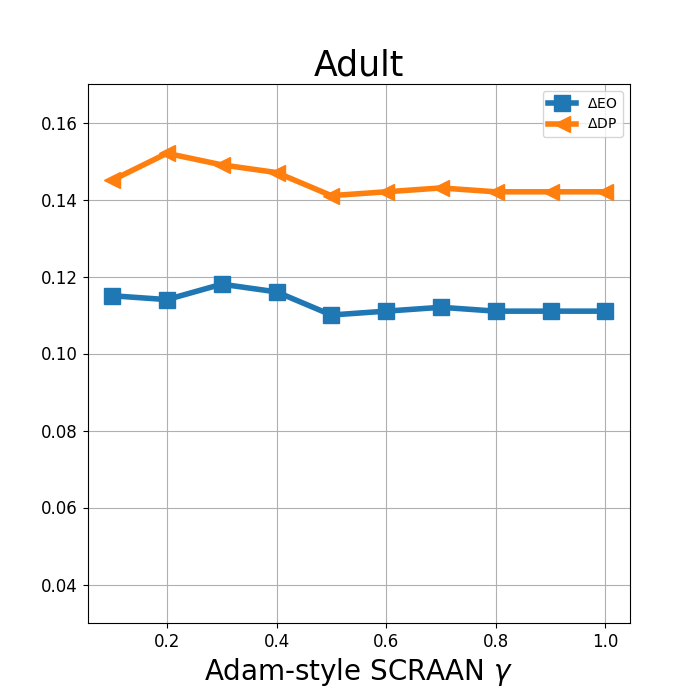}
     \includegraphics[width = 0.48\linewidth]{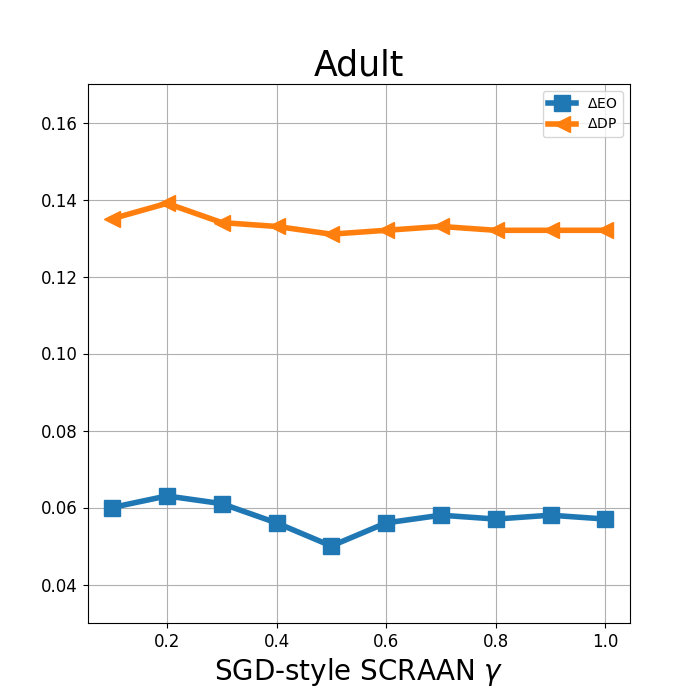}
     \caption{Robustness of $\gamma$.}
    \label{fig:gamma_robust}
\end{minipage}
\begin{minipage}[c]{0.34\textwidth}
\centering
    \includegraphics[width = 0.95\linewidth]{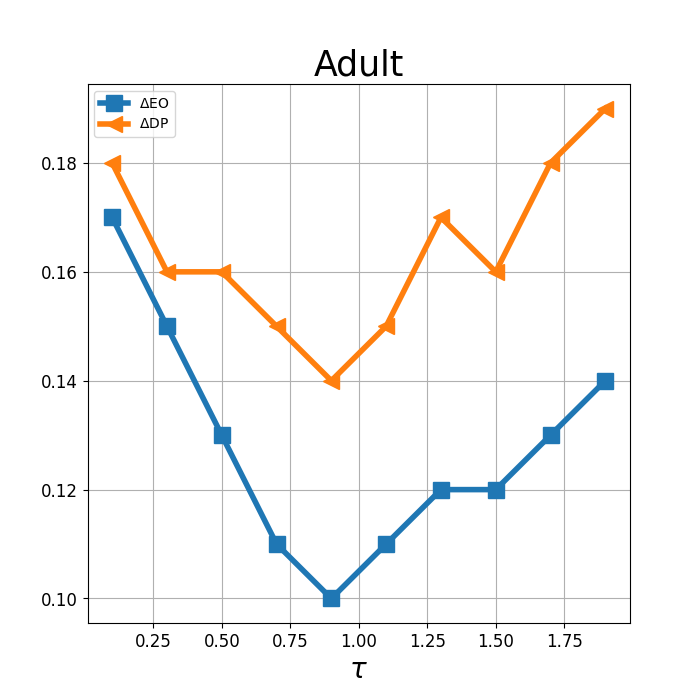}
    \caption{Influence of $\tau$}
 \label{fig:MEPS-ABA}
\end{minipage}
\hfill
\end{figure}

\subsection{More Experimental Results of SGD-style SCRAAN}
\label{sec:SGD-style}

Here we provide the SGD-style SCRAAN experimental results on the Adult dataset. We can see that our methods are better than the baselines which is consistent with Adam-style SCRAAN.  

 \begin{figure}[h]
     \centering
      \includegraphics[width = 0.3\linewidth]{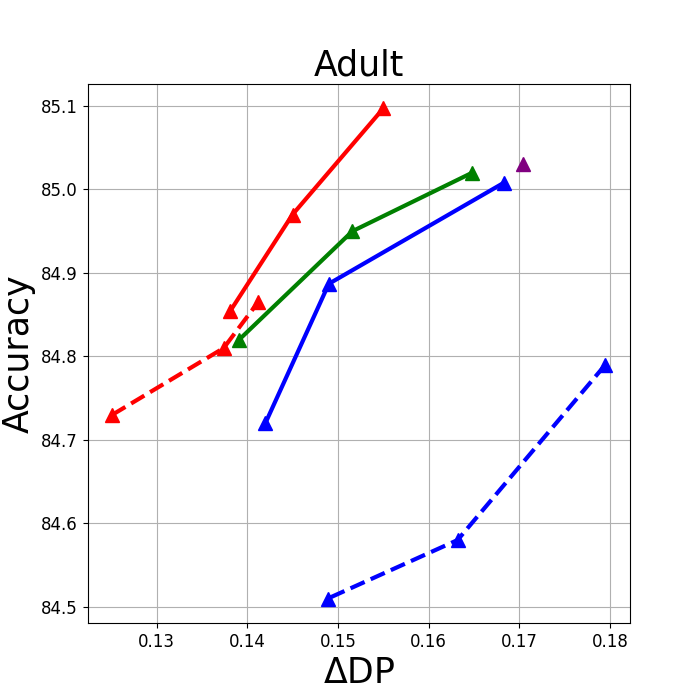}
     \includegraphics[width = 0.3\linewidth]{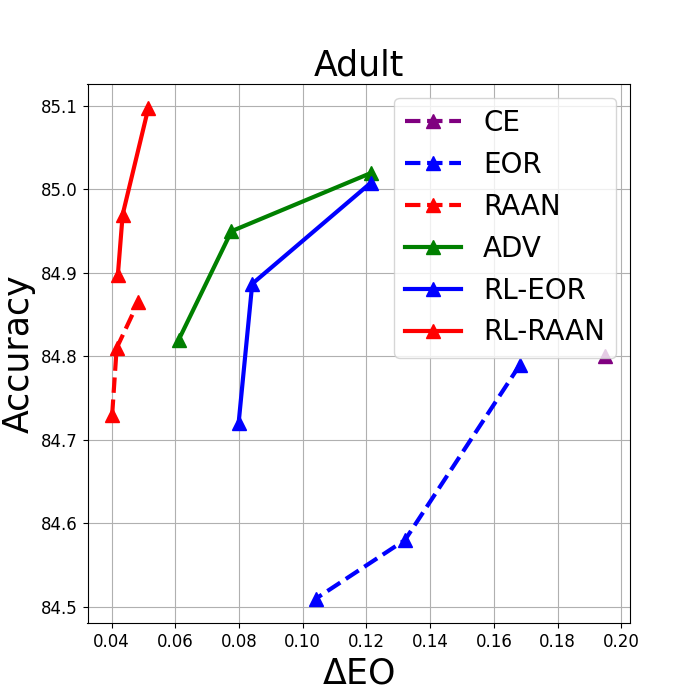}
     \caption{The $\Delta$DP and $\Delta$EO on the Adult dataset optimized by SGD-Style SCRAAN.}
  \label{fig:SGD-results}
 \end{figure}

\subsection{The Derivation of RAAN Objective, Equation~(\ref{eqn:RAAN}), in Section~\ref{sec:RAAN}}
\label{sec:deriv_RAAN}
Given the pairwise weights between each sample $i\sim\D$ and its ANN, i.e, equation~(\ref{eqn:l_i_lbl}),~(\ref{eqn:nb}). We have the following loss by averaging over all samples within the same protected groups, attributes and classes, wee have the following average neighbourhood robust loss.

Rewriting equation~(\ref{eqn:RAAN})
    \begin{align*}
      &\frac{1}{C}\sum\limits_{c=1}^C \frac{1}{|A|}\sum\limits_{a=1}^{|A|}\frac{1}{| \D^c_a|}  \sum\limits_{i=1}^{|\D^c_a|} \underbrace{\sum\limits_{j\in \P_i} p^{\text{AAN}}_{ij}\ell(\w;\x_j,c, a_j)}_{\ell_i^{\text{AAN}}}
    \end{align*}
where $p^{\text{AAN}}_{ij} =   \frac{\exp(\frac{\z_i^\top(\w_f)\z_j(\w_f)}{\tau})}{\sum\limits_{k\in \P_i}\exp(\frac{\z_i^\top(\w_f)\z_k(\w_f)}{\tau})}$. To start with, $p^{\text{AAN}}_{ij}$ is derived from the constraint robust pairwise objective in Equation~(\ref{eqn:nb}),
 \begin{align*}
     \max\limits_{\p^{\text{AAN}}_i\in \Delta^{|\P_i|}} \sum\limits_{j\in\P_i} p^{\text{AAN}}_{ij} \z_i(\w_f)^\top\z_j(\w_f) - \tau \text{KL}(\p^{\text{AAN}}_i\|\frac{1}{|\P_i|}),\bf 1 \in \R^{|\P_i|} 
 \end{align*}
 where $\Delta^{|\P_i|} = \{\p^{\text{AAN}}_i \in\R^{|\P_i|}, \sum_jp^{\text{AAN}}_{ij} = 1, 0\leq  p^{\text{AAN}}_{ij} \leq 1\}$. Note the expression of  $\text{KL}(\p^{\text{AAN}}_i\|\frac{1}{|\P_i|}) = \sum_j p_{ij} \log (|\P_i|p_{ij}) = \sum_jp_{ij} \log (p_{ij})+\log(|\P_i|)$. There are three constrains to handle, i.e., $\sum_jp^{\text{AAN}}_{ij} = 1,  p^{\text{AAN}}_{ij} \geq 0, \ \text{and}\ \ p^{\text{AAN}}_{ij}\leq 1. $
  Note that the constraint $p^{\text{AAN}}_{ij}\geq 0$ is enforced by the term $p^{\text{AAN}}_{ij}\log(p^{\text{AAN}}_{ij})$, otherwise the above objective will become infinity. As a result, the constraint $p^{\text{AAN}}_{ij}<1$ is automatically satisfied due to $\sum_jp^{\text{AAN}}_{ij}=1$ and $p^{\text{AAN}}_{ij}\geq0$. Hence, we only need to explicitly tackle the constraint $\sum_jp^{\text{AAN}}_{ij}=1$. To this end, we define the following Lagrangian function,
  \begin{align*}
 \tau L(\p^{\text{AAN}}_i, \mu) = -\sum_{j=1}^{|\P_i|} p_{ij}\z_i(\w_f)^\top\z_j(\w_f)  +  \tau(\log |\P_i| + \sum_{j=1}^{|\P_i|} p^{\text{AAN}}_{ij} \log (p^{\text{AAN}}_{ij})) + \mu(\sum_j p^{\text{AAN}}_{ij}- 1) 
\end{align*}
where $\mu$ is the Lagrangian multiplier for the constraint  $\sum_jp^{\text{AAN}}_{ij}=1$. The optimal solutions satisfy the KKT conditions:  
\begin{align*}
    &- \z_i(\w_f)^\top\z_j(\w_f)  +  \tau (\log (p^{\text{AAN}}_{ij}(\mathbf w)) + 1) + \mu  = 0, \\
    &\sum_jp^{\text{AAN}}_{ij}=1
\end{align*}
From the first equation, we can derive $p^{\text{AAN}}_{ij} \propto \exp(\frac{\z_i(\w_f)^\top\z_j(\w_f)}{\tau})$. Then according to the second equation, we can conclude that
\begin{align*}
 p^{\text{AAN}}_{ij} = \frac{\exp(\frac{\z_i(\w_f)^\top\z_j(\w_f)}{\tau})}{\sum_{j\in\P_i} \exp(\frac{\z_i(\w_f)^\top\z_j(\w_f)}{\tau})} =  \frac{\exp(\frac{\z_i^\top(\w_f)\z_j(\w_f)}{\tau})}{\sum\limits_{k\in \P_i}\exp(\frac{\z_i^\top(\w_f)\z_k(\w_f)}{\tau})}
\end{align*}

Next, we derive the second equivalence in the robust objective, Equation~(\ref{eqn:RAAN}).

\begin{align*}
 \text{RAAN(\w)} :=\ \   &
 \frac{1}{C}\sum\limits_{c=1}^C \frac{1}{A}\sum\limits_{a=1}^{A} \frac{1}{| \D_a^c|}\sum\limits_{i=1}^{|\D^c_a|} \bigg(\sum\limits_{j\in\P_i} \underbrace{\frac{\exp(\frac{\z_i^\top(\w_f)\z_j(\w_f)}{\tau})}{\sum\limits_{k\in\P_i}\exp(\frac{\z_i^\top(\w_f)\z_k(\w_f)}{\tau})}}_{p^{\text{AAN}}_{ij}}\ell(\w;\x_j, c, a)\bigg ) \\
:=\ \ &\frac{1}{CA}\sum\limits_{c=1}^C \sum\limits_{j\in\D^c}\underbrace{\bigg(\sum\limits_{i\in\P_j} \frac{1}{| \D_{a_i}^{y_i}|} \frac{\exp(\frac{\z_i^\top(\w_f)\z_j(\w_f)}{\tau})}{\sum\limits_{k\in \P_i}\exp(\frac{\z_i^\top(\w_f)\z_k(\w_f)}{\tau})}\bigg) }_{p^{\text{AAN}}_{j}}\ell(\w;\x_j, y_j, a_j) \\
:=\ \  &\frac{1}{CA}\sum\limits_{j\in \D}\underbrace{\bigg(\sum\limits_{i\in\P_j} \frac{1}{|\P_j|} \frac{\exp(\frac{\z_i^\top(\w_f)\z_j(\w_f)}{\tau})}{\sum\limits_{k\in \P_i}\exp(\frac{\z_i^\top(\w_f)\z_k(\w_f)}{\tau})}\bigg) }_{p^{\text{AAN}}_{j}}\ell(\w;\x_j, y_j, a_j) \\
\Longleftrightarrow := & \ \  \frac{1}{CA}\sum\limits_{j=1}^{|\D|} p_j^{\text{AAN}}\ell(\w;\x_j, y_j,a_j) 
\end{align*}
We finish the derivation. Therefore, RAAN combines the information from the embedding space $p^{\text{AAN}}_{j}$ to promote a more uniform embedding of the classification head.

\subsection{Theoretical Analysis}
\label{sec:thm_drivation}

To derive the theoretical analysis, we write the pairwise $\text{RAAN}(\w)$, i.e, the first equivalence in Equation~(\ref{eqn:RAAN})
\begin{align*}
\text{RAAN(\w)} :=\ \   &
 \frac{1}{C}\sum\limits_{c=1}^C \frac{1}{A}\sum\limits_{a=1}^{A} \frac{1}{| \D_a^c|}\sum\limits_{i=1}^{|\D^c_a|} \bigg(\sum\limits_{j\in\P_i} \underbrace{\frac{\exp(\frac{\z_i^\top(\w_f)\z_j(\w_f)}{\tau})}{\sum\limits_{k\in\P_i}\exp(\frac{\z_i^\top(\w_f)\z_k(\w_f)}{\tau})}}_{p^{\text{AAN}}_{ij}}\ell(\w;\x_j, c, a)\bigg )
 \end{align*}
We write it as a general compositional form $R(\w)$,
\begin{align*}
    \text{R}(\w) = \frac{1}{n}\sum\limits_{\x_i\in \D} f(g_{\x_i}(\w)) = \E_{\x_i\in \D}[f(g_{\x_i}(\w))]
\end{align*}
where $f(g) = \frac{g_1}{g_2}$, and $g_{\x_i}(\w) =\frac{n}{AC}\E_{\x_j\in\D} [\exp(\frac{\z_i(\w_f)^\top\z_j(\w_f)}{\tau})\ell_j(\w)\I(\x_j\in\P_i), \exp(\frac{\z_i(\w_f)^\top\z_j(\w_f)}{\tau})\I(\x_j\in\P_i)]^\top=\frac{n}{AC}\E_{\x_j\in\P_i} [\exp(\frac{\z_i(\w_f)^\top\z_j(\w_f)}{\tau})\ell_j(\w), \exp(\frac{\z_i(\w_f)^\top\z_j(\w_f)}{\tau})]^\top$, $\forall\  \tau \neq 0$, and $\ell_j(\w) = \ell_j(\w;\x_j, y_j, a_j)$.
Our theoretical analysis follows the same framework as SOAP in~\citep{qi2021stochastic}. Next, we first introduce the assumptions and provide a lemma to guarantee that $R(\cdot)$ is smooth.

\begin{ass}
\label{ass:1}
Assume that
    (a) there exists $\Delta_1$ such that $R(\w_1) - \min_\w R(\w)\leq \Delta_1$;
    (b) there exist $M>0$ such that $\ell(\w; \x_j, y_j, a_j)\leq M$ and $\ell(\w; \x_j, y_j, a_j)$ is $C_l$-Lipscthiz continuous and $L_l$-smooth with respect to $\w$ for any $\x_j\in\mathcal D$;
    (c) there exists $V>0$ such that  $\E_{\x_j\sim \D}[\| g(\w; \x_i, \x_j) - g_{\x_i}(\w)\|^2] \leq V$, and  $\E_{\x_j\sim\D}[\| \nabla g(\w; \x_i, \x_j) - \nabla g_{\x_i}(\w)\|^2] \leq V$ for any $\x_i$. 
\end{ass}

\begin{lem}
\label{lem:1}
Suppose Assumption~\ref{ass:1} holds, $\tau\geq \tau_0$,
$\text{maxE}=$ $\max\{ \exp(1/\tau_0),\exp(-1/\tau_0)\}, \text{minE} = \min\{$ $\exp(1/\tau_0),\exp(-1/\tau_0)\}$, there exists $u_0\geq \frac{n\cdot\text{minE}}{|\P_i|AC}$, $u_1= \frac{nM\cdot \text{maxE}}{AC}$ and $u_2 = \frac{n\cdot\text{maxE}}{AC}$ such that $g_{\x_i}(\w)\in\Omega = \{\u\in\R^2, 0\leq [\u]_1\leq u_1, u_0\leq[\u]_2\leq u_2\}$, $\forall \x_i\in\D$. In addition, there exists $L > 0$ such that $R(\cdot)$ is $L$-smooth.
\end{lem}

\begin{proof}
We first prove the first part $g_i(\w)\in\Omega$. Due to the definition of $g_{\x_i}(\w) =\frac{n}{AC}\E_{\x_j\in\D} [\exp(\frac{\z_i(\w_f)^\top\z_j(\w_f)}{\tau})$ $\ell_j(\w)\I(\x_j\in\P_i), \exp(\frac{\z_i(\w_f)^\top\z_j(\w_f)}{\tau})\I(\x_j\in\P_i)]^\top$.
As $\z_i(\w_f) = \frac{F(\w_f,\x_i)}{\|F(\w_f,\x_i)\|}$, $-1\leq \z_i(\w_f)^\top\z_j(\w_f)\leq 1$,  $\min\{(\exp(\frac{-1}{\tau}), \exp(\frac{1}{\tau})\} \leq \exp(\frac{\z_i(\w_f)^\top\z_j(\w_f)}{\tau}) \leq \max\{(\exp(\frac{-1}{\tau}), \exp(\frac{1}{\tau})\}$
Therefore, $0\leq [g_{\x_i}(\w)]_1\leq\frac{n\max\{\exp(1/\tau_0),\exp(-1/\tau_0)\}M}{AC}$ and $\frac{n\min\{\exp(1/\tau_0),\exp(-1/\tau_0)\}}{|\P_i|AC} \leq [g_{\x_i}(\w)]_2 \leq \frac{n\max\{\exp(1/\tau_0),\exp(-1/\tau_0)\}}{AC}\ \ \ \ \forall i, j$.  To this end, we need to use the following Lemma~\ref{lem:R-smooth} and the proof will be presented.
\end{proof}

\begin{lem}\label{lem:R-smooth}
Let $L_f =  \frac{4(u_0 + u_1)}{u_0^3} , C_f = \frac{u_0+u_1}{u_0^2} , L_g =\frac{10n\text{maxE}}{AC}(C_l+L_l), C_g =\frac{n\text{maxE}}{AC}(C_\ell+2M) $, then $f(\u)$ is a $L_f$-smooth, $C_f$-Lipschit continuous function for any $\u\in\Omega$, and $\forall i\in [1,\cdots, n]$, $g_{\x_i}$ is a $L_g$-smooth, $C_g$-Lipschitz continuous function.
\end{lem}

\begin{equation}
    \begin{aligned}
     f(\u) & =\frac{[\u]_1}{[\u]_2}\quad, \nabla_{\u} f(\u)  = \bigg(\frac{1}{[\u]_2},-\frac{[\u]_1}{([\u]_2)^2}\bigg)^\top, \quad\nabla_{\u}^2 f(\u)= \left(\begin{array}{c} 0,-\frac{1}{([\u]_2)^2}\\
   -\frac{1}{([\u]_2)^2}, \frac{2[\u]_1}{([\u]_2)^3}
    \end{array}\right)
    \end{aligned}
\end{equation}
Due to the assumption that $\ell(\w;\x_i)$ is a $L_l$-smooth, $C_l$-Lipschitz continuous function, and $\|\z_i(\w_f)\|^2= 1, -1\leq \z_i^\top(\w_f)\z_(\w_f)\leq 1$, we have
\begin{equation}
    \begin{aligned}
&\| \nabla^2_\w g_i(\w) \|  = \|\frac{n}{AC|\P_i|}\sum\limits_{j=1}^{|\P_i|}[\exp(\frac{\z_i(\w_f)^\top\z_j(\w_f)}{\tau})\nabla^2_{\w}\ell_j(\w;\x_j,c_j,a_j)\\
& +2(\z_i(\w_f) +\z_j(\w_f))\exp(\frac{\z_i(\w_f)^\top\z_j(\w_f)}{\tau})\nabla_\w\ell_j(\w;\x_j,c_j,a_j) \\
&+((\z_i(\w_f) +\z_j(\w_f))^2  + 2)\exp(\frac{\z_i(\w_f)^\top\z_j(\w_f)}{\tau})\ell_j(\w;\x_j,c_j,a_j) \|] \\
&\overset{(a)}{\leq} \frac{n}{AC}\frac{1}{|\P_i|}\sum\limits_{j=1}^{|\P_i|} \|[\exp(\frac{\z_i(\w_f)^\top\z_j(\w_f)}{\tau})\nabla^2_{\w}\ell_j(\w;\x_j,c_j,a_j)\\
& +2(\z_i(\w_f) +\z_j(\w_f))\exp(\frac{\z_i(\w_f)^\top\z_j(\w_f)}{\tau})\nabla_\w\ell_j(\w;\x_j,c_j,a_j) \\
&+((\z_i(\w_f) +\z_j(\w_f))^2  + 2)\exp(\frac{\z_i(\w_f)^\top\z_j(\w_f)}{\tau})\ell_j(\w;\x_j,c_j,a_j) \| \\
&\leq \frac{n\text{maxE}}{AC} (L_l + 10C_l) \leq \frac{10n\text{maxE}}{AC}(C_l+L_l) = L_g
    \end{aligned}
\end{equation}
where $(a)$ applies the convexity of $\|\cdot\|$ and $\|a +b \|\leq \|a \| + \|b\|$. Similarly, the following equations hold in terms of the continuous of inner objective $g_{\x_i}$, 
\begin{equation}
    \begin{aligned}
   \| \nabla_\w g_{\x_i}(\w) \|& = \|\frac{n}{AC|\P_i|}\sum\limits_{j=1}^{|\P_i|}[\exp(\frac{\z_i(\w_f)^\top\z_j(\w_f)}{\tau})\nabla_{\w}\ell_j(\w;\x_j,c_j,a_j)\\
&+(\z_i(\w_f) +\z_j(\w_f))\exp(\frac{\z_i(\w_f)^\top\z_j(\w_f)}{\tau})\ell_j(\w;\x_j,c_j,a_j) \|] \\
&\leq \frac{n}{AC}(\text{maxE}C_\ell +2\text{maxE}M)) =\frac{n\text{maxE}}{AC}(C_\ell+2M)  = C_g  
    \end{aligned}
\end{equation}

\begin{equation}
    \begin{aligned}
    & \| \nabla f(\u) \| \leq  \sqrt{\frac{1}{[\u]_2^2} + \frac{[\u]_1^2}{[\u]_2^4}}
\leq \frac{u_0+u_1}{u_0^2} = C_f
\\
& \|\nabla^2 f(\u) \| \leq\sqrt{ \frac{2}{[\u]_2^4} + 4\frac{[\u]^2_1}{[\u]_2^6}} \leq \frac{4(u_0 + u_1)}{u_0^3} = L_f
    \end{aligned}
\end{equation}

 Since $P(\w) = \frac{1}{n}\sum_{\x_i\in\D}f(g_i(\w))$.  We first show $R_i(\w) = f(g_i(\w))$ is smooth. To see this, 
\begin{align*}
    &\|\nabla R_i(\w) - \nabla R_i(\w')\| = \|\nabla g_i(\w)^{\top}\nabla f(g_i(\w)) - \nabla g_i(\w')^{\top}\nabla f(g_i(\w'))\|\\
    & \leq \|\nabla g_i(\w)^{\top}\nabla f(g_i(\w)) -  \nabla g_i(\w')^{\top}\nabla f(g_i(\w))\|\\
    & + \|\nabla g_i(\w')^{\top}\nabla f(g_i(\w)) - \nabla g_i(\w')^{\top}\nabla f(g_i(\w'))\|\\
    & \leq C_fL_g \|\w - \w'\| + C_gL_fC_g\|\w - \w'\| = ( C_fL_g  + L_fC_g^2)\|\w - \w'\|. 
\end{align*}
Hence $R(\w)$ is also  $L=( C_fL_g  + L_fC_g^2)$-smooth.

\subsection{Proof of Theorem~\ref{thm:main-Adam} (SCRAAN with SGD-Style Update)}

\begin{lem}
\label{lem:lem-SGD}
With $\alpha\leq 1/2$, running $T$ iterations of SCRAAN (SGD-style) updates, we have
\begin{align*}
  \frac{\alpha}{2}\E[\sum_{t=1}^T\|\nabla R(\w_t)\|^2 ]&\leq  \E[\sum_t( R(\w_t) - R(\w_{t+1})) ]+ \frac{\alpha C_1}{2}\E[\sum_{t=1}^T\|g_{i_t}(\w_t)-\u_{i_t}\|^2] + \alpha^2T C_2,
\end{align*}
where $i_t$ denotes the index of the sampled positive data at iteration $t$, $C_1$ and $C_2$ are proper constants. 
\end{lem}
Our key contribution is the following lemma that bounds the second term in the above upper bound. 
\begin{lem}\label{lem:cumulative_var}Suppose Assumption \ref{ass:1} holds, with $\u$ initialized inner objective stochastic estimator for every $\x_i\in\D$ we have
\begin{equation}
    \begin{aligned}
     \E[\sum_{t=1}^T\|g_{i_t}(\w_t)-\u_{i_t}\|^2]
     &\leq \frac{n V }{\gamma} + \gamma V T + 2\frac{n^2\alpha^2TC_3}{\gamma^2},
    \end{aligned}
\end{equation}
where $C_3$ is a proper constant. 
\end{lem}
\vspace*{-0.1in}
{\bf Remark:} The innovation of proving the above lemma is by grouping $\u_{i_t}, t=1, \ldots, T$ into $n$ groups corresponding to the $n$ samples AAN, and then establishing the recursion of the error $\|g_{i_t}(\w_t)-\u_{i_t}\|^2$ within each group, and then summing up these recursions together. %

\subsubsection{Proof of Lemma~\ref{lem:lem-SGD}}

\begin{proof}[Proof of Lemma~\ref{lem:lem-SGD}]
 To make the proof clear, we write $\nabla g_{i_t}(\w; \xi) = \nabla g(\w_t;\x_{i_t},\xi), \xi\sim\P_{i_t}$.
Let $\u_{i_t}$ denote the updated $\u$ vector at the $t$-th iteration for the selected positive data $i_t$. 
\begin{align*}
   & R(\w_{t+1}) - R(\w_t)\leq   \nabla R(\w_t)^{\top}(\w_{t+1} - \w_t) + \frac{L}{2}\|\w_{t+1} - \w_t\|^2\\
     & =  - \alpha\|\nabla R(\w_t)\|^2 +  \alpha \nabla R(\w_t)^{\top}(\nabla R(\w_t) - \nabla g_{i_t}^{\top}(\w_t; \xi)\nabla f(\u_{i_t})) + \frac{\alpha^2\|G(\w_t)\|^2L}{2}\\
    & \leq   - \alpha\|\nabla R(\w_t)\|^2+ \alpha \nabla R(\w_t)^{\top}(\nabla R(\w_t) - \nabla g_{i_t}^{\top}(\w_t; \xi)\nabla f(\u_{i_t})) + \alpha^2 C_2
\end{align*}
where {$C_2 = \|G(\w_t)\|^2L/2 \leq C_g^2C_f^2L/2 $.} Taking expectation on both sides, we have
\begin{align*}
    \E_t[R(\w_{t+1})]&\leq \E_t[R(\w_t) + \nabla R(\w_t)^{\top}(\w_{t+1} - \w_t) + \frac{L}{2}\|\w_{t+1} - \w_t\|^2]\\
    & =\E_t[ R(\w_t) - \alpha\|\nabla R(\w_t)\|^2  + \alpha \nabla R(\w_t)^{\top}(\nabla R(\w_t) - \nabla g_{i_t}(\w_t; \xi)^{\top}\nabla f(\u_{i_t}))] + \alpha^2 C_2\\
    & = R(\w_t) - \alpha\|\nabla R(\w_t)\|^2 + \alpha \nabla R(\w_t)^{\top}(\E_t[\nabla R(\w_t) - \nabla g_{i_t}(\w_t; \xi)^{\top}\nabla f(\u_{i_t})]) + \alpha^2 C_2
\end{align*}
where $\E_t$ means taking expectation over $i_t, \xi$ given $\w_t$.\\
\noindent Noting that $\nabla R(\w_t) = \E_{i_t, \xi}[\nabla g_{i_t}(\w_t; \xi)^{\top}\nabla f(g_{i_t}(\w_t))]$, where $i_t$ and $\xi$ are independent.
\begin{align*}
   & \E_t[R(\w_{t+1})]- R(\w_t)\\ &\leq - \alpha\|\nabla R(\w_t)\|^2 + \alpha \nabla R(\w_t)^{\top}(\E_{t}[\nabla g_{i_t}(\w_t; \xi)^\top\nabla f(g_{i_t}(\w_t))]- \E_t[\nabla g_{i_t}(\w_t; \xi)^{\top}\nabla f(\u_{i_t})]) + \alpha^2 C_2\\
    & = - \alpha\|\nabla R(\w_t)\|^2 + \E_t[\alpha \nabla R(\w_t)^{\top}(\nabla g_{i_t}(\w_t; \xi)^{\top}\nabla f(g_{i_t}(\w_t))- \nabla g_{i_t}(\w_t; \xi)^{\top}\nabla f(\u_{i_t}))] + \alpha^2 C_2\\
    &\overset{(a)}{\leq}  - \alpha\|\nabla R(\w_t)\|^2 +  \E_t[\frac{\alpha}{2}\| \nabla R(\w_t)\|^2 + \frac{\alpha}{2}\|\nabla g_{i_t}(\w_t; \xi)^{\top} \nabla f(g_{i_t}(\w_t))- \nabla g_{i_t}(\w_t; \xi)^{\top}\nabla f(\u_{i_t}))\|^2 + \alpha^2 C_2\\
    &\overset{(b)}{\leq}  - \alpha\|\nabla R(\w_t)\|^2 + \E_t[\frac{\alpha}{2}\| \nabla R(\w_t)\|^2 + \frac{\alpha C_1}{2}\|g_{i_t}(\w_t)-\u_{i_t}\|^2 + \alpha^2 C_2\\
    &=  - (\alpha - \frac{\alpha}{2})\|\nabla R(\w_t)\|^2  + \frac{\alpha C_1}{2}\E_t[\|g_{i_t}(\w_t)-\u_{i_t}\|^2] + \alpha^2 C_2\\
\end{align*}
where the equality (a) is due to $ab\leq a^2/2 +b^2/2$ and the inequality $(b)$ uses the factor $\|\nabla g_{i_t}(\w_t; \xi)\|\leq C_l$ and $\nabla f$ is $L_f$-Lipschitz continuous for $\u, g_i(\w)\in\Omega$ and $C_1 = C^2_lC^2_f$.
Hence we have,
\begin{align*}
   \frac{\alpha}{2}\|\nabla R(\w_t)\|^2 &\leq  R(\w_t) - \E_t[R(\w_{t+1}) ]+ \frac{\alpha C_1}{2}\E_t[\|g_{i_t}(\w_t)-\u_{i_t}\|^2] + \alpha^2 C_2\\
\end{align*}
Taking summation and expectation over all randomness, we have
\begin{align*}
  \frac{\alpha}{2}\E[\sum_{t=1}^T\|\nabla R(\w_t)\|^2 ]&\leq  \E[\sum_t( R(\w_t) - R(\w_{t+1})) ]+ \frac{\alpha C_1}{2}\E[\sum_{t=1}^T\|g_{i_t}(\w_t)-\u_{i_t}\|^2] + \alpha^2 C_2 T\\
\end{align*}
\end{proof}

\subsubsection{Proof of Lemma~\ref{lem:cumulative_var}}

Let $i_t$ denote the selected data $i_t$ at $t$-th iteration. We will divide $\{1,\ldots, T\}$ into $n$ groups with the $i$-th group given by $\mathcal T_i = \{t^i_1, \ldots, t^i_{k}\ldots, \}$, where $t^i_{k}$ denotes the iteration that the $i$-th index data is selected  at the $k$-th time for updating $\u$. Let us define $\phi(t): [T]\rightarrow [n]\times [T]$ that maps the selected data into its group index and within group index, i.e, there is an one-to-one correspondence between index $t$ and selected data $i$ and its index within $\mathcal T_i$. Below, we use notations $a^k_i$ to denote $a_{t^i_k}$. Let $T_i =|\mathcal T_i|$. Hence, $\sum_{i=1}^{n_+}T_i = T$.

\begin{proof}[Proof of Lemma~\ref{lem:cumulative_var}] To prove Lemma~\ref{lem:cumulative_var}, we first introduce another lemma that establishes a recursion for $\|\u_{i_t} - g_{i_t}(\w_t)\|^2$, whose proof is presented later.
\begin{lem}
\label{lem:var=B=1}
By the updates of SCRAAN Adam-style or SGD-style with a sample $\x_i\in\D,\ \text{and}, \ \xi\in \P_i$, the following equation holds for $\forall \ t\in 1,\cdots, T$
\begin{equation}
    \begin{aligned}
     \E_t[\|\u_{i_t} - g_{i_t}(\w_t)\|^2]&\overset{\phi(t)}{=} \E_t[\|\u^{k}_{i} - g_{i}(\w^k_i)\|^2]\\
     & \leq (1-\gamma) \|\u_i^{k-1} - g_i(\w^{k-1}_i)\|^2 + \gamma^2V +\gamma^{-1}\alpha^2n^2C_3
    \end{aligned}
\end{equation}
where $\E_t$ denotes the conditional expectation conditioned on history before $t^i_{k-1}$. 
\end{lem}

Then, by mapping every $i_t$ to its own group and make use of Lemma~\ref{lem:var=B=1}, we have
\begin{equation}
\label{eqn:recur-g}
    \begin{aligned}
     \E[\sum_{k=0}^{K_i}\|\u_i^{k} - g_i^{k}(\w^{k}_i)\|^2]
     &\leq \E\left[\frac{[\|\u_i^{0} - g_i(\w^{0}_i)\|^2] }{\gamma} + \gamma V T_i + \gamma^{-2}n^2C_3\alpha^2T_i\right]
    \end{aligned}
\end{equation}
where $\u_i^0$ is the initial vector for $\u_i$, which can be computed by a mini-batch averaging estimator of  $g_i(\w_{0})$. 
Thus
\begin{align*}
    \E[\sum_{t=1}^T\|g_{i_t}(\w_t)-\u_{i_t}\|^2] & \overset{\phi(t)}{=} \E[\sum\limits_{i= 1}^{n}\sum_{k=0}^{K_i}\|\u_i^{k} - g_i^{k}(\w^{k}_i)\|^2] \\
    &\leq \sum\limits_{i=1}^{n}\Big \{  \frac{[\|\u_i^{0} - g_i^{0}(\w^{0}_i)\|^2] }{\gamma} + \gamma V \E[T_i] + \gamma^{-2}n^2C_3\alpha^2\E[T_i]\Big \} \\
   &  \leq \frac{nV}{\gamma} + \gamma VT +  \frac{n^2\alpha^2TC_3}{\gamma^2}
\end{align*}

\end{proof}

\subsubsection{Proof of Lemma~\ref{lem:var=B=1}}

\begin{proof}
We first introduce the following lemma, whose proof is presented later. 
\begin{lem}
\label{lem:bound-stale}
Suppose the sequence generated in the training process using the positive sample $i$ is $\{ \w^i_{i_{1}}, \w^i_{i_{2}}, ..$ $..,\w^i_{i_{T_i}}\}$, where $0<i_{1}<i_2<\cdots < i_{T_i}\leq T$, then $\E_{|i_k}[i_{k+1} - i_{k}] \leq n_+,\text{and}, \E_{|i_k}[(i_{k+1} - i_{k})^2] \leq 2n^2, \forall k$. 
\end{lem}
Define   $\widetilde{g}_{i_t}(\w_t) = g(\w_t, \x_{i_t}, \xi)$.
Let $\prod_{\Omega}(\cdot): \R^2 \rightarrow \Omega$ denotes the projection operator. By the updates of $\u_{i_t}$, 
we have $\u_{i_t} =\u^k_i = \prod_{\Omega}[(1-\gamma)\u^{k-1}_i + \gamma \tilde{g}_{i_t}(\w_t)]$. 

\begin{align*}
     &\E_t[\|\u_{i_{t}} - g_{i_t}(\w_t)\|^2] 
     \overset{\phi(t)}{=} \E[\|\u^{k}_{i} - g_{i}(\w_i^k)\|^2]\\
     & = \E_t[\|\prod _{\Omega}((1-\gamma)\u^{k-1}_i + \gamma \widetilde{g}_i(\w^k_{i})) - \prod_{\Omega}(g_{i}(\w_t))\|^2] \\ 
     &\leq  \E_t[\|((1-\gamma)\u^{k-1}_i + \gamma \widetilde{g}_i(\w^k_{i})- g_{i}(\w_t)\|^2] \\
     &\leq  \E_t[\|((1-\gamma)(\u^{k-1}_i - g_i(\w^{k-1}_i)) + \gamma (\widetilde{g}_i(\w^k_{i})- g_{i}(\w^k_i)) + (1-\gamma)(g_i(\w^{k-1}_i) - g_{i}(\w^k_i))\|^2] \\
     &\leq  \E_t[\|((1-\gamma)(\u^{k-1}_i - g_i(\w^{k-1}_i)) + (1-\gamma)(g_i(\w^{k-1}_i) - g_{i}(\w^k_i))\|^2] + \gamma^2 V \\
     &\leq [(1-\gamma)^2(1+\gamma) \|\u_i^{k-1} - g_i(\w^{k-1}_i)\|^2] + \gamma^2V + \frac{(1+\gamma)(1-\gamma)^2}{\gamma}C_g\E[\|\w^k_i-\w_i^{k-1}\|^2]\\
     &\leq [(1-\gamma) \|\u_i^{k-1} - g_i(\w^{k-1}_i)\|^2] + \gamma^2V + \gamma^{-1}\alpha^2C_g\E_t[\|\sum_{t=t^i_{k-1}}^{t^i_{k}-1}\nabla g_{i_t}(\w_t ;\xi)\nabla f(\u_{i_t})\|^2]\\
     &\leq [(1-\gamma) \|\u_i^{k-1} - g_i(\w^{k-1}_i)\|^2] + \gamma^2V + \gamma^{-1}\alpha^2C_g\E_t[(t^i_k - t^i_{k-1})^2]C_g^2C_f^2)]\\
     &\overset{(a)}{\leq} \E[(1-\gamma) \|\u_i^{k-1} - g_i(\w^{k-1}_i)\|^2] + \gamma^2V +2\gamma^{-1}\alpha^2n^2C_g^3C_f^2 \\
     & \leq [(1-\gamma) \|\u_i^{k-1} - g_i(\w^{k-1}_i)\|^2] + \gamma^2V +\gamma^{-1}\alpha^2n^2C_3
   \end{align*}

where the inequality (a) is due to that $t^i_k - t^i_{k-1}$ is a geometric distribution random variable with $p=1/n$, i.e., $\E_{|t^i_{k-1}}[(t^i_k - t^i_{k-1})^2]\leq 2/p^2=2n^2$, by Lemma~\ref{lem:bound-stale}.
The last equality hold by defining $C_3 = 2C_g^3C_f^2$.

\end{proof}

\subsubsection{Proof of Lemma~\ref{lem:bound-stale}}
\begin{proof} Proof of Lemma~\ref{lem:bound-stale}.
Denote the random variable $\Delta_k = i_{k+1} - i_k$ that represents the iterations that the $i$th positive sample has been randomly selected for the $k+1$-th time conditioned on $i_k$.
Then $\Delta_k$ follows a Geometric distribution such that $\Pr(\Delta_k = j) = (1-p)^{j-1}p$, where $p = \frac{1}{n}$, $j = 1,2,3,\cdots$. As a result, $\E[\Delta_k|i_k] = 1/p = n$.
$\E[\Delta_k^2|i_k] = \text{Var}(\Delta_k) + \E[\Delta_k|i_k]^2= \frac{1-p}{p^2} + \frac{1}{p^2}\leq \frac{2}{p^2} = 2n^2$.
\end{proof}

\subsection{Proof of
Theorem~\ref{thm:main-Adam} (SCRAAN with Adam-Style Update)}
\begin{proof}

We first provide two useful lemmas, whose proof are presented later. 
\begin{lem}
\label{lem:update-Adam-B=D}
Assume assumption~\ref{ass:1} holds
\begin{equation}
\begin{aligned}
\|\w_{t+1} - \w_t\|^2 \leq \alpha^2d(1-\eta_2)^{-1}(1-\tau)^{-1}
\end{aligned}
\end{equation}
where $d$ is the dimension of $\w$, $\eta_1 < \sqrt{\eta_2} < 1$, and $\tau := \eta_1^2/\eta_2$.
\end{lem}

\begin{lem}
\label{lem:lem-Adam}
With $c =  (1+(1-\eta_1)^{-1})\epsilon^{-\frac{1}{2}}C_g^2L_f^2$, running $T$ iterations of SOAP (Adam-style) updates, we have
\begin{equation}
\label{eqn:thm-sum-T-1}
    \begin{aligned}
    & \sum\limits_{t=1}^T\frac{\alpha(1-\eta_1)(\epsilon + C_g^2C_f^2)^{-1/2}}{2}\| \nabla R(\w_t)\|^2 \leq \E[\V_1] - \E[\V_{T+1}]\\
     &+2\eta_1 L\alpha^2T d(1-\eta_1)^{-1}(1-\eta_2)^{-1}(1-\tau)^{-1}+L\alpha^2Td(1-\eta_2)^{-1}(1-\tau)^{-1} \\
& + 2(1-\eta_1)^{-1}\alpha C_g^2C_f^2\sum\limits_{i'=1}^d((\epsilon + \hat{v}^{i'}_{0})^{-1/2})  +c\alpha\sum\limits_{t=1}^T \E_t[  \|g_{i_t}(\w_t) - \u_{i_t}\|^2] \\
    \end{aligned}
\end{equation}
where $\V_{t+1} =P(\w_{t+1}) - c_{t+1} \langle \nabla P(\w_{t}), D_{t+1} h_{t+1} \rangle $.

\end{lem}

According to Lemma~\ref{lem:lem-Adam} and plugging Lemma~\ref{lem:cumulative_var} into equation~(\ref{eqn:thm-sum-T-1}), we have
\begin{equation}
\label{eqn:thm-sum-T-3}
    \begin{aligned}
    &\sum\limits_{t=1}^T\frac{\alpha(1-\eta_1)(\epsilon + C_g^2C_f^2)^{-1/2}}{2}\| \nabla R(\w_t)\|^2 \\
    & \leq  \E[\V_{1}]- \E[\V_{T+1}] +2\eta_1 L\alpha^2T d(1-\eta_1)^{-1}(1-\eta_2)^{-1}(1-\tau)^{-1}+L\alpha^2dT(1-\eta_2)^{-1}(1-\tau)^{-1} \\
& + 2c\alpha C_g^2C_f^2\sum\limits_{i'=1}^d(\epsilon + \hat{v}^{i'}_{0})^{-1/2} + c\alpha (\frac{nV}{\gamma} + 2\gamma V T + \frac{2C_gn^2C_3\alpha^2T}{\gamma^2})
\\
    \end{aligned}
\end{equation}

Let $\eta' = (1-\eta_2)^{-1}(1-\tau)^{-1}, \eta^{''}  = (1-\eta_1)^{-1}(1-\eta_2)^{-1}(1-\tau)^{-1}$, and $\widetilde{\eta} = (1-\eta_1)^{-2}(1-\eta_2)^{-1}(1-\tau)^{-1}$. As $(1-\eta_1)^{-1}\geq 1,(1-\eta_2)^{-1}\geq 1 $, then $\widetilde{\eta} \geq \eta^{''}\geq\eta'\geq 1$.

Then by rearranging terms in Equation~(\ref{eqn:thm-sum-T-3}), dividing $\alpha T (1+\eta_1)(\epsilon +C_g^2C_f^2)^{-1/2}$ on both sides and suppress constants, $C_g, L_g, C_3, L, C_f, L_f, V, \epsilon$ into big $O$, we get
\begin{equation}
\label{eqn:grad_s_bound}
    \begin{aligned}
    \frac{1}{T}\sum\limits_{t=1}^T\| \nabla R(\w_t)\|^2
    &\leq
    \frac{1}{\alpha T(1-\eta_1)} O\Big (\E[\V_1] - \E[\V_{T+1}] +\eta^{''}
    \eta_1\alpha^2 Td + \eta^{'}
    \alpha^2 Td  + \alpha \sum\limits_{i'=1}^d(\epsilon + \hat{v}^{i'}_{0})^{-1/2}\\
    &+\frac{c\alpha n}{\gamma} + c\alpha\gamma T + \frac{c\alpha^3n^2T}{\gamma^{2}} \Big ) \\
    &\overset{(a)}{\leq}
    \frac{1}{\alpha T(1-\eta_1)} O\Big (\E[\V_1] - \E[\V_{T+1}] +\eta^{''}
    \eta_1\alpha^2 Td + \eta^{'}
    \alpha^2 Td  + \alpha d(\epsilon + C_fC_g)^{-1/2}\\
    &+\frac{c\alpha n}{\gamma} + c\alpha\gamma T + \frac{c\alpha^3n^2T}{\gamma^{2}} \Big ) \\
    & \overset{(b)}{\leq} \frac{\widetilde{\eta}}{\alpha T} O\Big (\E[\V_1] - [\V_{T+1}] + (1+
    \eta_1)\alpha^2 Td  + \alpha d +\frac{c\alpha n}{\gamma} + c\alpha\gamma T + \frac{c\alpha^3n^2T}{\gamma^2} \Big )
    \end{aligned}
\end{equation}
where the inequality $(a)$ is due to $\hat{v}_0^{i'} = G^{i'}(\w_0)^2 \leq \|G(\w_0)\|^2\leq  C^2_fC^2_g$. The last inequality $(b)$ is due to  $\widetilde{\eta} \geq \eta^{''}\geq\eta'\geq 1$.

Moreover, by the definition of $\V$ and $\w_0 = \w_1$, we have 
\begin{equation}
\label{eqn:L_0_T+1}
    \begin{aligned}
      \E[\V_1] & = R(\w_1) -c_{1} \langle \nabla R(\w_{0}), D_1 h_1 \rangle  \leq R(\w_1) + c_{1}\|\nabla R(\w_{0})\|\|\w_{1}  -\w_{0}\|\frac{1}{\alpha} 
     =R(\w_1) \\
    - \E[\V_{T+1}] &\leq -R(\w_{T+1}) + c_{T+1} \langle \nabla R(\w_{T}), D_T h_T \rangle \\
    &\leq -\min_\w R(\w) +  c_{T+1}\|\nabla R(\w_{t-1})\|\|\w_{t+1}  -\w_{t}\|\frac{1}{\alpha}\\
    &\overset{(a)}{\leq} -\min_\w R(\w) + (1-\eta_1)^{-1}\alpha\sqrt{d}(1-\eta_2)^{-1/2}(1-\tau)^{-1/2} \\
    &\overset{(b)}{\leq} -\min_\w R(\w) + \widetilde{\eta}\sqrt{d}\alpha
    \end{aligned}
\end{equation}
where the inequality $(a)$ is due to Lemma~\ref{lem:update-Adam-B=D} and $c_{T+1}\leq (1-\eta_1)^{-1}\alpha$ in equation~(\ref{eqn:thm3-c}). The inequality $(b)$ is due to $(1-\eta_1)^{-1}(1-\eta_2)^{-1/2}(1-\tau)^{-1/2}
\leq (1-\eta_1)^{-1}(1-\eta_2)^{-1}(1-\tau)^{-1} \leq \eta^{''}\leq \widetilde{\eta}$. \\
\noindent
Thus $  \E[\V_1] - \E[\V_{T+1}] \leq P(\w_1) - \min_\w P(\w) + \widetilde{\eta}\sqrt{d}\alpha \leq \Delta_1 +\widetilde{\eta}\sqrt{d}\alpha$  by combining equation~(\ref{eqn:grad_s_bound}) and ~(\ref{eqn:L_0_T+1}).
\\
\noindent
Then we have
\begin{equation}
    \begin{aligned}
    \frac{1}{T}\sum\limits_{t=1}^T\| \nabla R(\w_t)\|^2 &\leq
    \widetilde{\eta} O\Big ( \frac{\Delta_1+\widetilde{\eta}\sqrt{d}\alpha}{\alpha T} + (1+\eta_1)\alpha d  +  \frac{d}{T} + \frac{nc}{T\gamma} +  c\gamma +\frac{\alpha^2n^2}{\gamma^2}  \Big )  \\
   &\overset{(a)}{ \leq} \widetilde{\eta} O \Big (  \frac{\Delta_1 n^{2/5}}{T^{2/5}} +\frac{\widetilde{\eta}\sqrt{d}}{T}  + \frac{(1+\eta_1)d}{n^{2/5}T^{3/5}} + \ \frac{d}{T} +  \frac{cn^{3/5}}{T^{3/5}} +2\frac{cn^{2/5}}{T^{2/5}}  \Big )\\
   &\overset{(b)}{\leq} O(\frac{n^{2/5}}{T^{2/5}})
    \end{aligned}
\end{equation}
The inequality $(a)$ is due to $\gamma= \frac{n^{2/5}}{T^{2/5}}$, $\alpha = \frac{1}{n^{2/5}T^{3/5}}$. In inequality $(b)$, we further compress the $\Delta_1$, $\eta_1$, $\widetilde{\eta}$, $c$ into big $O$ and $\gamma \leq 1 \rightarrow n^{2/5}\leq T^{2/5}$.

\end{proof}

\subsubsection{Proof of
Lemma~\ref{lem:update-Adam-B=D}}

\begin{proof} This proof is following the proof of Lemma 4 in ~\citep{chen2021solving}.

 Choosing $\eta_1 < 1$ and defining $\tau =\frac{\eta_1^2}{\eta_2}$, with the Adam-style (Algorithm~\ref{alg:3}) updates of SOAP that $h_{t+1} = \eta_1 h_{t} + (1-\eta_1)G(\w_t)$, we can verify for every dimension $l$,
 \begin{equation}
 \label{eqn:lemma4-1}
     \begin{aligned}
         |h^l_{t+1}| &= |\eta_1 h^l_t + (1-\eta_1)G^l(\w_t)| \leq \eta_1|h^l_t| +|G^l(\w_t)| \\
         & \leq \eta_1(\eta_1|h^l_{t-1}| +|G^l(\w_{t-1})|) + |G^l(\w_t)| \\
         &\leq \sum\limits_{p=0}^t \eta_1^{t-p}|G^l(\w_p)| = \sum\limits_{p=0}^t \sqrt{\tau}^{t-p}\sqrt{\eta_2}^{t-p}|G^l(\w_p)| \\
         &\leq \Big(\sum\limits_{p=0}^t\tau^{t-p}\Big)^{\frac{1}{2}}\Big(\sum\limits_{p=0}^t\eta_2^{t-p}(G^l(\w_p))^2\Big)^{\frac{1}{2}}\\
         &\leq (1-\tau)^{-\frac{1}{2}}\Big ( \sum\limits_{p=0}^t \eta_2^{t-p}(G^l(\w_t) )^2 \Big )^{\frac{1}{2}}
     \end{aligned}
 \end{equation}
where $\w^l$ is the $l$th dimension of $\w$, the third inequality follows the Cauchy-Schwartz inequality.
For the $l$th dimension of $\hat{v}$, $\hat{v}^l_t$, first we have $\hat{v}_1^l \geq (1-\eta_2)(G^l(\w_1)^2)$. Then since
    \begin{align*}
    \hat{v}^l_{t+1} \geq \eta_t\hat{v}^l_t + (1-\eta_2)(G^l(\w_t))^2
    \end{align*}
by induction we have
\begin{equation}
 \label{eqn:lemma4-2}
    \begin{aligned}
    \hat{v}^l_{t+1} \geq (1-\eta_2)\sum\limits_{p=0}^t \eta_2^{t-p}(G^l(\w_t))^2
    \end{aligned}
\end{equation}
Using equation~(\ref{eqn:lemma4-1}) and equation~(\ref{eqn:lemma4-2}), we have
\begin{equation}
    \begin{aligned}
    |h^l_{t+1}|^2 &\leq (1-\tau)^{-1}\Big ( \sum\limits_{p=0}^t \eta_2^{t-p}(G^l(\w_t) )^2 \Big )\\
    &\leq (1-\eta_2)^{-1}(1-\tau)^{-1}\hat{v}^l_{t+1}
    \end{aligned}
\end{equation}
Then follow the  Adam-style update in Algorithm~\ref{alg:3}, we have
\begin{equation}
    \begin{aligned}
    \|\w_{t+1} - \w_t\|^2 = \alpha^2 \sum\limits_{l=1}^d (\epsilon + \hat{v}^l_{t+1})^{-1} |h^l_{t+1}|^2 \leq \alpha^2d(1-\eta_2)^{-1}(1-\tau)^{-1}
    \end{aligned}
\end{equation}
which completes the proof.
\end{proof}

\subsubsection{Proof of Lemma~\ref{lem:lem-Adam}}
\begin{proof}  To make the proof clear, we make some definitions the same as the proof of Lemma~\ref{lem:lem-SGD}. Denote by $\nabla g_{i_t}(\w_t; \xi) = \nabla g(\w_t; \x_{i_t},\xi), \xi\sim\P_{i_t}$, where $i_t$ is a positive sample randomly generated from $\D$ at $t$-th iteration, and $\xi$ is a random sample that generated from $\D$ at $t$-th iteration. It is worth to notice that $i_t$ and $\xi$ are independent. $\u_{i_t}$ denote the updated $\u$ vector at the $t$-th iteration for the selected positive data $i_t$.

\begin{align*}
    R(\w_{t+1})&\leq R(\w_t) + \nabla R(\w_t)^{\top}(\w_{t+1} - \w_t) + \frac{L}{2}\|\w_{t+1} - \w_t\|^2\\
    & \leq R(\w_t)-\alpha \nabla R(\w_t)^{\top}(D_{t+1}h_{t+1}) +  \alpha^2d(1-\eta_2)^{-1}(1-\tau)^{-1}L/2\\
\end{align*}
where $D_{t+1} = \frac{1}{\sqrt{\epsilon I + \hat{\v}_{t+1}}}$, $h_{t+1} =\eta_1 h_t + (1-\eta_1)  \nabla g_{i_t}^{\top}(\w_t; \xi)\nabla f(\u_{i_t})$ and the second inequality is due to Lemma~\ref{lem:update-Adam-B=D}. Taking expectation on both sides, we have
\begin{align*}
    \E_t[R(\w_{t+1})]&\leq R(\w_t)  \underbrace{- \E_t[ \nabla R(\w_t)^{\top}(D_{t+1}h_{t+1})]}_{\Upsilon}\alpha + \alpha^2d(1-\eta_2)^{-1}(1-\tau)^{-1}L
\end{align*}
where $\E_t[\cdot] = \E[\cdot | \F_t]$ implies taking expectation over $i_t, \xi$ given $\w_t$. 
In the following analysis, we decompose $\Upsilon$ into three parts and bound them one by one:

\begin{align*}
\Upsilon &=  -  \langle \nabla R(\w_t), D_{t+1}h_{t+1}\rangle = -\langle \nabla R(\w_t), D_th_{t+1}\rangle -\langle \nabla R(\w_t), (D_{t+1} - D_{t})h_{t+1} \rangle \\
& = -(1-\eta_1)\langle \nabla R(\w_t), D_{t}\nabla g_{i_t}(\w_t; \xi)^{\top}\nabla f(\u_{i_t})\rangle - \eta_1\langle \nabla R(\w_t), D_{t}h_t \rangle  \\
& - \langle \nabla R(\w_t), (D_{t+1} - D_{t})h_{t+1} \rangle\\
&= I_1^t + I_2^t + I_3^t
\end{align*}

Let us first bound $I_1^t$,
\begin{equation}
\label{eqn:thm3-I_1_1}
\begin{aligned}
\E_t[I_1^t] & \overset{(a)}{=}-(1-\eta_1)\langle \nabla R(\w_t), \E_{t}[D_{t}\nabla g_{i_t}(\w_t; \xi)^{\top}\nabla f(\u_{i_t})]\rangle  \\
& = -(1-\eta_1) \langle \nabla R(\w_t), \E_{t}[D_{t}\nabla g_{i_t}(\w_t; \xi)^{\top}\nabla f(g_{i_t}(\w_t))]\rangle\\
&+ (1-\eta_1)\langle \nabla R(\w_t), \E_{t}[D_{t}\nabla g_{i_t}(\w_t; \xi)^{\top}(\nabla f(\u_{i_t}) -\nabla f(g_{i_t}(\w_t))]\rangle \\
& \leq -(1-\eta_1)\|\nabla R(\w_t)\|^2_{D_t}\\
&+(1-\eta_1)\|D_{t}^{-1/2} \nabla R(\w_t)\|\E_t[\|D_{t}^{-1/2}\nabla g_{i_t}(\w_t; \xi)^{\top}(\nabla f(\u_{i_t}) -\nabla f(g_{i_t}(\w_t)))\|] \\
& \overset{(b)}{\leq} -(1-\eta_1)\|\nabla R(\w_t)\|^2_{D_t} +\frac{(1-\eta_1)\| \nabla R(\w_t) \|^2_{D_t} }{2}\\
&+ \frac{(1-\eta_1)\E_t[\|D_{t}^{-1/2}\nabla g_{i_t}(\w_t; \xi)^{\top}(\nabla f(\u_{i_t}) -\nabla f(g_{i_t}(\w_t)))\|^2]}{2} \\
 &\leq  -\frac{(1-\eta_1)}{2}\|\nabla R(\w_t)\|^2_{D_t} + \frac{ (1-\eta_1)}{2}\E_t[\|\nabla g_{i_t}(\w_t; \xi)^{\top}(\nabla f(\u_{i_t}) -\nabla f(g_{i_t}(\w_t))\|^2_{D_t}] \\
 &\overset{(c)}{\leq} -\frac{(1-\eta_1)}{2}(\epsilon + C_g^2C_f^2)^{-1/2}\|\nabla R(\w_t) \|^2 +\frac{1}{2}\epsilon^{-1/2}C^2_g L_f^2\E[\|g_{i_t}(\w_t) - \u_{i_t}\|^2] \\
\end{aligned}
\end{equation}
where equality $(a)$ is due to $\nabla R(\w_t)  = \E_{i_t, \xi}[\nabla g_{i_t}(\w_t; \xi)^{\top}\nabla f(g_{i_t}(\w_t))]$, where $i_t$ and $\xi$ are independent.
 The inequality $(b)$ is according to $ab\leq a^2/2  +b^2/2$. The last inequality $(c)$ is due to $\epsilon^{-1/2} \I\geq \|D_t \I\| = \|\frac{1}{\sqrt{\epsilon \I + \hat{v}_{t+1}}}\|\geq\| (\epsilon \I + C_g^2C_f^2)^{-1/2}\| = (\epsilon + C_g^2C_f^2)^{-1/2} \I$, $(1-\eta_1) \leq 1$ and 
\begin{equation}
\label{eqn:thm3-I_1}
\begin{aligned}
&\E_t[\|\nabla g_{i_t}(\w_t; \xi)^{\top}(\nabla f(\u_{i_t}) -\nabla f(g_{i_t}(\w_t)))\|^2_{D_t}] \\
& \leq \epsilon^{-1/2}C^2_g\E_t[\| \nabla f(\u_{i_t}) -\nabla f(g_{i_t}(\w_t)) \|_{\I}^2] \\
&\leq  \epsilon^{-1/2}C^2_g L_f^2\E_t[\|g_{i_t}(\w_t) - \u_{i_t}\|^2]
\end{aligned}
\end{equation}

For $I_2^t$ and $I_3^t$, we have
\begin{equation}
\label{eqn:thm3-I_2}
\begin{aligned}
\E_t[I_2^t]& = - \eta_1\langle \nabla R(\w_t) - \nabla R(\w_{t-1}), D_{t}h_t \rangle  - \eta_1 \langle \nabla R(\w_{t-1}), D_{t}h_t \rangle \\
&\leq \eta_1L\alpha^{-1}\|\w_t -\w_{t-1}\|^2 - \eta_1 \langle \nabla R(\w_{t-1}), D_{t}h_t \rangle \\
& = \eta_1L\alpha^{-1}\|\w_t -\w_{t-1}\|^2 +\eta_1 (I_1^{t-1}+I_2^{t-1}+I_3^{t-1}) \\
&\leq \eta_1L\alpha d(1-\eta_2)^{-1}(1-\tau)^{-1} +\eta_1 (I_1^{t-1}+I_2^{t-1}+I_3^{t-1})
\end{aligned}
\end{equation}
where the last equation applies Lemma~\ref{lem:update-Adam-B=D}.

\begin{equation}
\label{eqn:thm3-I_3}
\begin{aligned}
\E_t[I_3^t] &= - \langle\nabla R(\w_t), (D_{t+1} - D_{t})h_{t+1} \rangle = -\sum\limits_{i'=1}^d \nabla_{i'} R(\w_t)((\epsilon + \hat{v}^{i'}_{t})^{-1/2}-(\epsilon + \hat{v}^{i'}_{t+1})^{-1/2})h^{i'}_{t+1}\\
&\leq \|\nabla R(\w_t)\|\|h_{t+1}\|\sum\limits_{i'=1}^d((\epsilon + \hat{v}^{i'}_{t})^{-1/2}-(\epsilon + \hat{v}^{i'} _{t+1})^{-1/2})\\
& \leq C_g^2C_f^2\sum\limits_{i'=1}^d((\epsilon + \hat{v}^{i'}_{t})^{-1/2}-(\epsilon + \hat{v}^{i'}_{t+1})^{-1/2})
\end{aligned}
\end{equation}
By combining Equation~(\ref{eqn:thm3-I_1}),~(\ref{eqn:thm3-I_2}) and~(\ref{eqn:thm3-I_3}) together,
\begin{equation}
\label{eqn:thm3-I-recur}
\begin{aligned}
\E_t[I_1^t + I_2^t + I_3^t]  &\leq  -\frac{(1-\eta_1)}{2}(\epsilon  + C_g^2C_f^2)^{-1/2} \| \nabla R(\w_t)\|^2 +\frac{1}{2} \epsilon^{-1/2}C^2_g L_f^2\E_t[\|g_{i_t}(\w_t) - \u_{i_t}\|^2]  \\
&+\eta_1L\alpha d(1-\eta_2)^{-1}(1-\tau)^{-1} +\eta_1 (I_1^{t-1}+I_2^{t-1}+I_3^{t-1})\\
&+  C_g^2C_f^2\sum\limits_{i'=1}^d((\epsilon + \hat{v}^{i'}_{t})^{-1/2}-(\epsilon + \hat{v}^{i'}_{t+1})^{-1/2})
\end{aligned}
\end{equation}

Define the Lyapunov function
\begin{equation}
\label{eqn:thm3-Laypnov}
\begin{aligned}
\V_t =R(\w_t) - c_t \langle \nabla R(\w_{t-1}), D_t h_t \rangle \\
\end{aligned}
\end{equation}
where $c_t$ and $c$ will be defined later.

\begin{equation}
\begin{aligned}
& \E_t[\V_{t+1} - \V_{t}] \\
& = R(\w_{t+1}) - R(\w_t) - c_{t+1} \langle \nabla R(\w_{t}), D_{t+1}h_{t+1} \rangle +  c_{t} \langle \nabla R(\w_{t-1}), D_th_{t} \rangle \\
&\leq  - (c_{t+1}+\alpha) \langle \nabla R(\w_{t}), D_{t+1}h_{t+1} \rangle  + \frac{L}{2}\|\w_{t+1} - \w_t\|^2 +   c_{t} \langle \nabla R(\w_{t-1}), D_th_{t} \rangle
  \\
&= (c_{t+1} +\alpha)(I_1^t+I_2^t+I_3^t) + \frac{L}{2}\|\w_{t+1} - \w_t\|^2-c_{t}(I_1^{t-1} +I_2^{t-1}+I_3^{t-1})\\
&\overset{Eqn~ (\ref{eqn:thm3-I-recur})\  \text{and}\  Lemma~\ref{lem:update-Adam-B=D}}{\leq} -(\alpha +c_{t+1})\frac{(1-\eta_1)}{2}(\epsilon  + C_g^2C_f^2)^{-1/2}\| \nabla R(\w_t)\|^2  \\
&+(\alpha +c_{t+1})\eta_1 L\alpha d(1-\eta_2)^{-1}(1-\tau)^{-1} +\eta_1 (\alpha +c_{t+1})(I_1^{t-1}+I_2^{t-1}+I_3^{t-1})  \\
& + (\alpha +c_{t+1}) C_g^2C_f^2\sum\limits_{i'=1}^d((\epsilon + \hat{v}_{i'}^{t})^{-1/2} - (\epsilon + \hat{v}_{i'} ^{t+1})^{-1/2})  \\
& +\frac{L}{2}\alpha^2d(1-\eta_2)^{-1}(1-\tau)^{-1}-c_{t}(I_1^{t-1} +I_2^{t-1}+I_3^{t-1}) + \frac{\epsilon^{-1/2}C_g^2L_f^2(\alpha +c_{t+1})}{2} \|g_{i_t}(\w_t) -\u_{i_t}\|^2\\
\end{aligned}
\end{equation}
By setting $\alpha_{t+1} \leq \alpha_t = \alpha$, $c_t = \sum\limits_{p=t}^\infty(\prod\limits_{j=t}^p\eta_1)\alpha_j$, and $c =  (1+(1-\eta_1)^{-1})\epsilon^{-\frac{1}{2}}C_g^2L_f^2$, we have
\begin{equation}
\label{eqn:thm3-c}
\begin{aligned}
 c_t\leq (1-\eta_1)^{-1}\alpha_t,\  \frac{2(\alpha + c_{t+1})}{\alpha}\beta \epsilon^{-1/2}C_g^2L_f^2 \leq  c\beta, \ 
\eta_1(\alpha + c_{t+1}) = c_{t}
\end{aligned}
\end{equation}
As a result,  $\eta_1 (\alpha +c_{t+1})(I_1^{t-1}+I_2^{t-1}+I_3^{t-1})  - c_t(I_1^{t-1}+I_2^{t-1}+I_3^{t-1}) = 0$
\begin{equation}
\label{eqn:L_t}
\begin{aligned}
\E_t[\V_{t+1} - \V_{t}]
&\leq -(\alpha +c_{t+1})\frac{(1-\eta_1)}{2}(\epsilon  + C_g^2C_f^2)^{-1/2}\| \nabla R(\w_t)\|^2  \\
&+(\alpha +c_{t+1})\eta_1 L\alpha d(1-\eta_2)^{-1}(1-\tau)^{-1}+\frac{L}{2}\alpha^2d(1-\eta_2)^{-1}(1-\tau)^{-1} \\
& + (\alpha+c_{t+1}) C_g^2C_f^2\sum\limits_{i'=1}^d((\epsilon + \hat{v}_{i'}^{t})^{-1/2}-(\epsilon + \hat{v}_{i'} ^{t+1})^{-1/2})  \\
& + \frac{(\alpha +c_{t+1})}{2} \epsilon^{-1/2}C_g^2L_f^2  \|g_{i_t}(\w_t) -\u_{i_t}\|^2   \\
& \leq -\alpha \frac{(1-\eta_1)}{2}(\epsilon  + C_g^2C_f^2)^{-1/2}\| \nabla R(\w_t)\|^2 \\
& +2\eta_1 L\alpha^2T d(1-\eta_1)^{-1}(1-\eta_2)^{-1}(1-\tau)^{-1}+\frac{L}{2}T\alpha^2d(1-\eta_2)^{-1}(1-\tau)^{-1} \\
& + 2(1-\eta_1)^{-1}\alpha C_g^2C_f^2\sum\limits_{i'=1}^d((\epsilon + \hat{v}_{i'}^{t})^{-1/2}-(\epsilon + \hat{v}_{i'} ^{t+1})^{-1/2})  
 +\frac{c\alpha}{4}\sum\limits_{t=1}^T \E_t[  \|g_{i_t}(\w_t) - \u_{i_t}\|^2]
\end{aligned}
\end{equation}
where the last inequality is due to equation~(\ref{eqn:thm3-c}) such that we have $2(\alpha + c_{t+1}) \epsilon^{-1/2}C_g^2L_f^2 \leq  c\alpha$, and $\alpha + c_{t+1} \leq 2(1-\eta_1)^{-1}\alpha$.
\\
Then by rearranging terms, and taking summation from $1,\cdots, T$ of equation~(\ref{eqn:L_t}), we have
\begin{equation}
\label{eqn:thm-sum-T-1-b}
    \begin{aligned}
    & \sum\limits_{t=1}^T\alpha\frac{(1-\eta_1)}{2}(\epsilon  + C_g^2C_f^2)^{-1/2}\| \nabla R(\w_t)\|^2 \leq \sum\limits_{t=1}^T \E_t[\V_{t} - \V_{t+1}]\\
     &+2\eta_1 L\alpha^2T d(1-\eta_1)^{-1}(1-\eta_2)^{-1}(1-\tau)^{-1}+LT\alpha^2d(1-\eta_2)^{-1}(1-\tau)^{-1} \\
& + 2(1-\eta_1)^{-1}\alpha C_g^2C_f^2 \sum\limits_{t=1}^T\sum\limits_{i'=1}^d((\epsilon + \hat{v}_{i'}^{t})^{-1/2}-(\epsilon + \hat{v}_{i'} ^{t+1})^{-1/2})  +c\alpha\sum\limits_{t=1}^T \E_t[  \|g_{i_t}(\w_t) - \u_{i_t}\|^2] \\
& \leq \E[\V_1] -\E[\V_{T+1}]\\
     &+2\eta_1 L\alpha^2T d(1-\eta_1)^{-1}(1-\eta_2)^{-1}(1-\tau)^{-1}+LT\alpha^2d(1-\eta_2)^{-1}(1-\tau)^{-1} \\
& + 2(1-\eta_1)^{-1}\alpha C_g^2C_f^2 
\sum\limits_{i'=1}^d((\epsilon + \hat{v}^{i'}_{0})^{-1/2})  +c\alpha\sum\limits_{t=1}^T \E_t[  \|g_{i_t}(\w_t) - \u_{i_t}\|^2]
    \end{aligned}
\end{equation}

By combing with Lemma~\ref{lem:cumulative_var},
We finish the proof.
\end{proof}

\end{document}